\documentclass{article}

\usepackage[nonatbib, preprint]{neurips_2021}

\usepackage[utf8]{inputenc} % allow utf-8 input
\usepackage[T1]{fontenc}    % use 8-bit T1 fonts
\usepackage{hyperref}       % hyperlinks
\usepackage{url}            % simple URL typesetting
\usepackage{booktabs}       % professional-quality tables
\usepackage{amsfonts}       % blackboard math symbols
\usepackage{nicefrac}       % compact symbols for 1/2, etc.
\usepackage{microtype}      % microtypography
\usepackage{xcolor}         % colors

\usepackage{geometry}
\usepackage{verbatim}
\usepackage{enumitem}
\usepackage{amsmath}
\usepackage{graphicx}
\usepackage{amsthm}
\usepackage{amssymb}
\usepackage{color}
\usepackage{thmtools, thm-restate}
\usepackage{bbm}
\usepackage{amsthm}
\usepackage{tikz-cd}
\usepackage{subcaption}

\usepackage{algorithm}
\usepackage{algpseudocode}

\usepackage{color-edits}
\addauthor{vs}{red}
\addauthor{dr}{blue}

\makeatletter
\renewcommand{\paragraph}{%
  \@startsection{paragraph}{4}%
  {\z@}{2.25ex \@plus 1ex \@minus .2ex}{-1em}%
  {\normalfont\normalsize\bfseries}%
}
\makeatother

\newcommand\restr[2]{{% we make the whole thing an ordinary symbol
  \left.\kern-\nulldelimiterspace % automatically resize the bar with \right
  #1 % the function
  \vphantom{\big|} % pretend it's a little taller at normal size
  \right|_{#2} % this is the delimiter
  }}

\newtheorem{theorem}{Theorem}[section]
\newtheorem{lemma}[theorem]{Lemma}

\newtheorem{corollary}[theorem]{Corollary} 

\theoremstyle{definition}
\newtheorem{definition}[theorem]{Definition}

\newtheorem{assumption}[theorem]{Assumption}

\theoremstyle{plain}

\DeclareMathOperator{\Var}{Var}

\newcommand{\norm}[1]{\left \lVert #1 \right \rVert}

\mathchardef\mhyphen="2D

\newtheoremstyle{break}
  {\topsep}{\topsep}%
  {}{}%
  {\bfseries}{}%
  {\newline}{}%
\theoremstyle{break}

\DeclareMathOperator*{\argmin}{argmin}

\newcommand{\RR}{\mathbb{R}}

\newcommand{\EE}{\mathbb{E}}

\DeclareMathOperator{\Cov}{Cov}

\DeclareMathOperator{\poly}{poly}

\DeclareMathOperator{\TV}{TV}
\newcommand{\xor}{\triangle}

\newcommand{\Gradient}{\nabla}

\theoremstyle{remark}
\newtheorem{remark}{Remark}

\newcommand{\ignore}[1]{{}}

\makeatletter
\renewcommand{\paragraph}{%
  \@startsection{paragraph}{4}%
  {\z@}{1.25ex \@plus 1ex \@minus .2ex}{-1em}%
  {\normalfont\normalsize\bfseries}%
}
\makeatother

\title{Robust Generalized Method of Moments:\\ A Finite Sample Viewpoint}

\author{Dhruv Rohatgi \thanks{This work was partially done while the first author was an intern at Microsoft Research New England.} \\ MIT \\ \texttt{drohatgi@mit.edu} \And Vasilis Syrgkanis \\ Microsoft Research, New England \\ \texttt{vasy@microsoft.com}}

\begin{document}
\maketitle

\begin{abstract}
For many inference problems in statistics and econometrics, the unknown parameter is identified by a set of moment conditions. A generic method of solving moment conditions is the Generalized Method of Moments (GMM). However, classical GMM estimation is potentially very sensitive to outliers. Robustified GMM estimators have been developed in the past, but suffer from several drawbacks: computational intractability, poor dimension-dependence, and no quantitative recovery guarantees in the presence of a constant fraction of outliers. In this work, we develop the first computationally efficient GMM estimator (under intuitive assumptions) that can tolerate a constant $\epsilon$ fraction of adversarially corrupted samples, and that has an $\ell_2$ recovery guarantee of $O(\sqrt{\epsilon})$. To achieve this, we draw upon and extend a recent line of work on algorithmic robust statistics for related but simpler problems such as mean estimation, linear regression and stochastic optimization. As two examples of the generality of our algorithm, we show how our estimation algorithm and assumptions apply to instrumental variables linear and logistic regression. Moreover, we experimentally validate that our estimator outperforms classical IV regression and two-stage Huber regression on synthetic and semi-synthetic datasets with corruption.
\end{abstract}

\section{Introduction}

Econometric and causal inference methodologies are increasingly being incorporated in automated large scale decision systems. Inevitably these systems need to deal with the plethora of practical issues that arise from automation. One important aspect is being able to deal with corrupted or irregular data, either due to poor data collection, the presence of outliers, or adversarial attacks by malicious agents. Even more classical applications of econometric methods in social science studies, can greatly benefit from robust inference so as not to draw conclusions solely driven by a handful of samples, as was recently highlighted in \cite{broderick2021automatic}.

Recent work in statistical machine learning has enabled robust estimation for regression problems and more generally estimation problems that reduce to the minimization of a stochastic loss. However, many estimation methods in causal inference and econometrics do not fall under this umbrella. A more general statistical framework that encompasses the most widely used estimation techniques in econometrics and causal inference is the framework of estimating models defined via \emph{moment conditions}. In this paper we offer a robust estimation algorithm that extends prior recent work in robust statistics to this more general estimation setting.

For a family of distributions $\{\mathcal{D}_\theta: \theta \in \Theta\}$, identifying the parameter $\theta$ is often equivalent to solving 
\begin{equation} \EE_{X \sim \mathcal{D}_\theta}[g(X, \theta)] = 0, \label{eq:moment-conditions} \end{equation}
for an appropriate problem-specific vector-valued function $g$. This formalism encompasses such problems as linear regression (with covariates $X$, response $Y$, and moment $g((X, Y), \theta) = X(Y - X^T\theta)$) and instrumental variables linear regression (with covariates $X$, response $Y$, instruments $Z$, and moment $g((X,Y,Z), \theta) = Z(Y - X^T\theta)$).

Under simple identifiability assumptions, moment conditions are statistically tractable, and can be solved by the \emph{Generalized Method of Moments} (GMM) \cite{hansen1982large}. Given independent observations $X_1,\dots,X_n \sim \mathcal{D}_\theta$, the GMM estimator is \[\hat{\theta} = \argmin_{\theta \in \Theta} \left(\frac{1}{n} \sum_{i=1}^n g(X_i, \theta)\right)^T W \left(\frac{1}{n} \sum_{i=1}^n g(X_i, \theta)\right)\] for a positive-definite weight matrix $W$. Of course, for general functions $g$, finding $\hat{\theta}$ (the global minimizer of a potentially non-convex function) may be computationally intractable. Under stronger assumptions, all approximate \emph{local} minima of the above function are near the true parameter, in which case the GMM estimator is efficiently approximable. For instrumental variables (IV) linear regression, these assumptions follow from standard non-degeneracy assumptions.

Due to its flexibility, the GMM estimator is widely used in practice (or heuristic variants, in models where it is computationally intractable). Unfortunately, like most other classical estimators in statistics, the GMM estimator suffers from a lack of robustness: a single outlier in the observations can arbitrarily corrupt the estimate.

\paragraph{Robust statistics} Initiated by Tukey and Huber in the 1960s, robust statistics is a broad field studying estimators which have provable guarantees even in the presence of outliers \cite{huber2004robust}. Outliers can be modelled as samples from a heavy-tailed distribution, or even as adversarially and arbitrarily corrupted data. Classically, robustness of an estimator against arbitrary outliers is measured by breakdown point (the fraction of outliers which can be tolerated without causing the estimator to become unbounded \cite{hampel1971general}) and influence (the maximum change in the estimator under an infinitesimal fraction of outliers \cite{hampel1974influence}). These metrics have spurred development and study of numerous statistical estimators which are often used in practice to mitigate the effect of outliers (e.g. Huber loss for mean estimation, linear regression, and other problems \cite{huber1992robust}).

Unfortunately, classical robust statistics suffers from a number of limitations due to emphasis on statistical efficiency and low-dimensional statistical problems. In particular, until the last few years, most high-dimensional statistical problems lacked robust estimators satisfying the following basic properties (see e.g. \cite{diakonikolas2019robust} for discussion in the setting of learning Gaussians and mixtures of Gaussians):
\begin{enumerate}
    \item Computational tractability (i.e. evading the curse of dimensionality)
    \item Robustness to a constant fraction of arbitrary outliers
    \item Quantitative error guarantees without dimension dependence.
\end{enumerate}

In a revival of robust statistics within the field of theoretical computer science, estimators with the above properties have been developed for various fundamental problems in high-dimensional statistics, including mean and covariance estimation \cite{diakonikolas2019robust, diakonikolas2017being}, linear regression \cite{diakonikolas2019efficient, bakshi2021robust}, and stochastic optimization \cite{diakonikolas2019sever}. However, practitioners in econometrics and applied statistics often employ more sophisticated inference methods such as GMM and IV regression, for which computationally and statistically efficient robust estimators are still lacking.

\paragraph{Our contribution} In this work, we address the aforementioned lack. Extending the \textsc{Sever} algorithm for robust stochastic optimization \cite{diakonikolas2019sever}, we develop a computationally efficient and provably robust GMM estimator under intuitive deterministic assumptions about the uncorrupted data. We instantiate this estimator for two special cases of GMM---instrumental variables linear regression and instrumental variables logistic regression---under distributional assumptions about the covariates, instruments, and responses (and in fact our algorithm also applies to the IV generalized linear model under certain conditions on the link function).

We corroborate the theory with experiments solving IV linear regression on corrupted synthetic and semi-synthetic data, which demonstrate that our algorithm outperforms non-robust IV as well as Huberized IV.

\paragraph{Techniques and Relation to [DKKLSS19]}
Our robust GMM algorithm builds upon the \textsc{Sever} algorithm and framework introduced in \cite{diakonikolas2019sever} for stochastic optimization. In this section, we briefly outline the relation. The \textsc{Sever} algorithm robustly finds an approximate critical point for the empirical mean of input functions $f_1,\dots,f_n: \RR^d \to \RR$, i.e. for convex functions, approximately and robustly solves \[\frac{1}{n} \sum_{i=1}^n \Gradient f_i(w^*) = 0.\]
The approach is to alternate between (a) finding an approximate critical point $\hat{w}$ of the current sample set, and (b) filtering the sample set by $\Gradient f_i(\hat{w})$, until convergence (i.e. when no samples are filtered out). Filtering ensures that at convergence, the mean of $\Gradient f_i(\hat{w})$ over the current sample set (which is small by criticality) is near the mean over the uncorrupted samples, so $\hat{w}$ is an approximate critical point for the uncorrupted samples, as desired.

Any moment condition which is the gradient of some function can be interpreted as a critical-point finding problem, and solved in the above way. An example is linear regression, where the moment $g(w) = X(Y - X^T W)$ is the gradient of the squared-loss $f(w) = \norm{Y - X^T w}_2^2$. However, $g(w) = Z(Y - X^T w)$ is not a gradient, so IV linear regression cannot directly be solved by \textsc{Sever}. In general, we need a way to robustly find an approximate solution to \[\frac{1}{n}\sum_{i=1}^n g(w^*) = 0.\]
Our approach is to alternate approximately minimizing \[\norm{\frac{1}{|S|} \sum_{i \in S} g(w)}_2^2,\] where $S$ is the current sample set, with a filtering step. However, it is not sufficient to filter by $g(\hat{w})$, because the minimization step does not necessarily output $\hat{w}$ for which $\frac{1}{|S|} \sum_{i \in S} g(w)$ is small (unlike for \textsc{Sever}, where $g = \Gradient f$, and so an approximate zero of $\frac{1}{|S|} \sum_{i \in S} \Gradient f_i(w)$ can always be found, for an arbitrary set of functions $\{f_i\}_{i \in S}$).

To fix this, we introduce a second filtering step based on $\Gradient g$. Under an identifiability condition for the uncorrupted samples (which is needed even in the absence of corruption), we show that the above situation, where $\frac{1}{|S|} \sum_{i \in S} g(w)$ is large, can be detected by the gradient filtering step, so that at convergence the empirical moment is in fact small.

\paragraph{Further related work} The generalized method of moments and instrumental variables regression have indeed been studied in the context of robust statistics \cite{amemiya1982two, freue2013natural, krasker1986two, ronchetti2001robust}. However, the resulting estimators face the same nearly ubiquitous issues described above. For instance, \cite{amemiya1982two} presents a variant of two-stage least squares which uses least absolute deviations. The resulting estimator performs well under the metric of bounded influence, but an arbitrary outlier can still cause arbitrary changes in the estimator. The estimator proposed by \cite{freue2013natural} modifies the closed-form solution to IV linear regression using robust mean and covariance estimators. These have attractive theoretical properties but are computationally intractable, and the heuristics by which they are implemented in practice have no associated theoretical guarantees. The robust GMM estimator presented in \cite{ronchetti2001robust} has bounded influence but is not robust to a constant fraction of outliers.

\section{Preliminaries}

For random variables $\{\xi_i\}_{i \in S}$ indexed by a set $S$, we use the notation $\EE_S[\xi_i]$ for the sample expectation $\frac{1}{|S|} \sum_{i \in S} \xi_i$. Similarly, if $\xi_i$ are scalars, then we define the sample variance $\Var_S(\xi_i) = \EE_S(\xi_i - \EE_S \xi_i)^2$. If $\xi_i$ are vectors then we define the sample covariance matrix $\Cov_S(\xi_i) = \EE_S (\xi_i - \EE_S \xi_i)(\xi_i - \EE_S \xi_i)^T$. A random vector $X$ is $(4,2,\tau)$-hypercontractive if $\EE(\langle X,u\rangle)^4 \leq \tau (\EE(\langle X,u\rangle)^2)^2$ for all vectors $u$.

\begin{definition}
For a closed set $\mathcal{H}$, a function $f: \mathcal{H} \to \RR$, and $\gamma > 0$, a $\gamma$-approximate critical point of $f$ (in $\mathcal{H}$) is some $x \in \mathcal{H}$ such that for any vector $v$ with $x+\delta v \in \mathcal{H}$ for arbitrarily small $\delta>0$, it holds that $v \cdot \Gradient f(x) \geq -\gamma \norm{v}_2$.
\end{definition}

\begin{definition}
For a closed set $\mathcal{H}$, a $\gamma$-approximate learner $\mathcal{L}_\mathcal{H}$ is an algorithm which, given a differentiable function $f: \mathcal{H} \to \RR$ returns a $\gamma$-approximate critical point of $f$.
\end{definition}

\begin{definition}
The (unscaled) \emph{logistic function} $G: \RR \to \RR$ is defined by $G(x) = 1/(1 + e^{-x})$.
\end{definition}

% define deterministic expectation

%\subsection{Generalized Method of Moments}

%In the classical GMM problem, we seek to solve a moment equation $\EE g(w) = 0$ given $n$ i.i.d. sample moments $g_1,\dots,g_n$. Statistically, a solution $w^*$ is identifiable if $\EE g(w)$ has a unique zero at $w^*$, and is bounded away from zero outside a neighborhood of $w^*$. Computationally, one approach to estimating $w^*$ is to compute an approximate stationary point of\[\norm{\frac{1}{n}\sum_{i=1}^n g_i(w)}_2^2.\] Under a ``global identifiability" condition that $\EE \Gradient g(w)^T \cdot \EE g(w)$ is only small when $w$ is near $w^*$, such an approximate stationary point is a good estimate of $w^*$. Moreover, approximate stationary points are efficiently computable for bounded-below functions.

\paragraph{Outline}
In Section~\ref{section:assumptions}, we describe the robust GMM problem, and we describe deterministic assumptions on a set of corrupted sample moments, under which we'll be able to efficiently estimate the parameter which makes the uncorrupted moments small. In Section~\ref{section:filter}, we describe a key subroutine of our robust GMM algorithm, which is commonly known in the literature as \emph{filtering}. In Section~\ref{section:gmm-algorithm}, we describe the robust GMM algorithm and prove a recovery guarantee under the assumptions from Section~\ref{section:assumptions}. In Section~\ref{section:iv}, we apply this algorithm to instrumental variable linear and logistic regression, proving that under reasonable stochastic assumptions on the uncorrupted data, arbitrarily $\epsilon$-corrupted moments from these models satisfy the desired deterministic assumptions with high probability. Finally, in Section~\ref{section:experiments}, we evaluate the performance of our algorithm on two corrupted datasets.

\section{Robust GMM Model}\label{section:assumptions}

In this section, we formalize the model in which we will provide a robust GMM algorithm. Classically, the goal of GMM estimation is to identify $\theta \in \Theta$ given data $X_1,\dots,X_n \sim \mathcal{D}_\theta$, using the moment condition $\EE_{X \sim \mathcal{D}_\theta}[g(X,\theta)] = 0$. We consider the added challenge of the \emph{$\epsilon$-strong contamination model}, in which an adversary is allowed to inspect the data $X_1,\dots,X_n$ and replace $\epsilon n$ samples with arbitrary data, before the algorithm is allowed to see the data. This corruption model encompasses most reasonable sources of outliers.

For our main theorem, we do not make stochastic assumptions about $\{\mathcal{D}_\theta: \theta \in \Theta\}$. Instead, we make deterministic assumptions (strong identifiability, boundedness, and so forth) about the moments of the corrupted data $\{g_i(\theta) := g(X_i, \theta)\}_{i=1}^n$ (for IV linear and logistic regression, we will prove that the assumptions hold with high probability under reasonable distributional assumptions).

Concretely, since only $\epsilon n$ samples were corrupted, there is a set of $(1-\epsilon)n$ uncorrupted samples. This set is unknown, but all we need is that it exists. More specifically, we assume that there exists a large subset of the data, of size at least $(1-\epsilon)n$, which satisfies finite-sample analogues of various distributional identifiability and boundedness conditions.

%Suppose that the data we are given are $n$ functions $g_1,\dots,g_n: \RR^d \to \RR^p$. These functions can be interpreted as independent samples from some distribution over functions, except $\epsilon n$ of the samples are replaced by adversarially chosen ``outlier" functions. However, to prove the guarantees of our estimation algorithm, we only make deterministic assumptions about the given data. 

%In the stochastic formulation, we know that these functions are drawn i.i.d. as some random function $g: \RR^d \to \RR^p$, but then $\epsilon n$ of the samples are replaced by adversarial outlier functions. We seek to approximately find the solution $w^*$ to the moment equation \[\EE g(w) = 0.\]

%To this end, we identify deterministic conditions on the corrupted sample set $g_1,\dots,g_n$ under which recovery is possible.

\begin{assumption}\label{assumption:gmm}
Given differentiable moments $g_1,\dots,g_n: \RR^d \to \RR$, a corruption parameter $\epsilon>0$, well-conditionedness parameters $\lambda$ and $L$, a Lipschitzness parameter $L_g$, and a noise level parameter $\sigma^2$, there is a set $I_\text{good} \subseteq [n]$ with $|I_\text{good}| \geq (1-\epsilon)n$ (the ``uncorrupted samples''), a vector $w^* \in \RR^d$ (the ``true parameter''), and a radius $R_0 \geq \norm{w^*}_2$ with the following properties:
\begin{itemize}[leftmargin=1em]
    \item \textbf{Strong identifiability.} $\sigma_\text{min}(\EE_{I_\text{good}} \Gradient g(w^*)) \geq \lambda$
    \item \textbf{Bounded-variance gradient.} $\EE_{I_\text{good}} (u^T\Gradient g(w^*)v)^2 \leq L^2$ for all unit-vectors $u \in \RR^p$, $v \in \RR^d$
    \item \textbf{Bounded-variance noise.} $\EE_{I_\text{good}} (v\cdot g(w^*))^2 \leq \sigma^2 L$ for all unit vectors $v$
    \item \textbf{Well-specification.} $\norm{\EE_{I_\text{good}} g(w^*)}_2 \leq \sigma\sqrt{L\epsilon}$
    \item \textbf{Lipschitz gradient.} $\norm{\EE_{I_\text{good}}\Gradient g(w) - \EE_{I_\text{good}}\Gradient g(w^*)}_\text{op} \leq L_g\norm{w-w^*}_2$ for all $w \in B_{2R_0}(0)$
    \item \textbf{Stability of gradient.} $R_0 < \lambda/(8L_g)$.
\end{itemize}
\end{assumption}

The stability of the gradient condition essentially states that the radius of the ball containing $w^*$ is sufficiently small that $\EE \Gradient g(w)$ cannot change much in this ball. Note that if the gradient is constant in $w$, as for IV linear regression, then the Lipschitz gradient assumption is satisfied with $L_g = 0$, and the stability assumption is vacuous. For non-linear moment problems, such as our logistic IV regression problem, this condition requires that the $\ell_2$-norm of the parameters be sufficiently small, such that the logits do not approach the flat region of the logistic function, a condition that is natural to avoid loss of gradient information and extreme propensities.

We state several consequences of Assumption~\ref{assumption:gmm} which will be used later (proof in Appendix~\ref{lemma:assumption-corollaries-appendix}).

\begin{lemma}\label{lemma:assumption-corollaries}
Under Assumption~\ref{assumption:gmm}, the following bounds hold for all $w \in B_{2R_0}(0)$:
\begin{itemize}
  \setlength\itemsep{-.2em}

    \item $\EE_{I_\text{good}}(u^T \Gradient g(w)v)^2 \leq 2L^2$ for all unit vectors $u \in \RR^p$ and $v \in \RR^d$
    \item $\Cov_{I_\text{good}}(g(w)) \preceq 2\sigma^2 L + 4L^2 \norm{w-w^*}_2^2I$
    \item $\sigma_\text{min}(\EE_{I_\text{good}} \Gradient g(w)) \geq \lambda/2$
    \item $\norm{\EE_{I_\text{good}} g(w)}_2 \leq \sigma\sqrt{L\epsilon} + 2L\norm{w-w^*}_2$
    \item $\norm{\EE_{I_\text{good}} \Gradient g(w)}_\text{op} \leq L\sqrt{2}$
\end{itemize}
\end{lemma}

\section{The \textsc{Filter} Algorithm}\label{section:filter}

In many robust statistics algorithms, an important subroutine is a \emph{filtering} algorithm for robust mean estimation. In this section we describe a filtering algorithm used in numerous prior works, including e.g. \cite{diakonikolas2019sever}. Given a set of vectors $\{\xi_i: i \in S\}$ and a threshold $M$, the algorithm returns a subset, by thresholding outliers in the direction of largest variance. Formally, see Algorithm~\ref{alg:filter}.

\begin{algorithm}
    \caption{\textsc{Filter}}
    \label{alg:filter}
    \begin{algorithmic}[1]
        \Procedure{Filter}{$\{\xi_i: i \in S\}, M$}
            \State $\hat{\xi} \gets \EE_S[\xi_i]$, $\Cov_S(\xi_i) = \EE_S[(\xi_i - \hat{\xi})(\xi_i - \hat{\xi})^T]$
            \State $v \gets $ largest eigenvector of $\Cov_S(\xi_i)$
            \State $\tau_i \gets (v \cdot (\xi_i - \hat{\xi}))^2$ for $i \in S$
            \If{$\frac{1}{|S|} \sum_{i \in S} \tau_i \leq 24M$}
                \State \textbf{return} $S$
            \Else
                \State Sample $T \gets \text{Unif}([0,\max \tau_i])$
                \State \textbf{return} $S \setminus \{i \in S: \tau_i > T\}$
            \EndIf
        \EndProcedure
    \end{algorithmic}
\end{algorithm}

%Let $S \subseteq [n]$. For any function $\xi(g)$, the algorithm \textsc{Filter} is defined as follows, with input vectors $\{\xi_i = \xi(g_i): i \in S\}$ and threshold $M$.
%\begin{itemize}
%    \item Let $\hat{\xi} = \EE_S[\xi_i]$ and $\Cov_S(\xi_i) = \EE_S[(\xi_i-\hat{\xi})(\xi_i-\hat{\xi})^T]$
%    \item Let $v$ be the largest eigenvector of $\Cov_S(\xi_i)$
%    \item Compute $\tau_i = (v\cdot (\xi_i - \hat{\xi}))^2$
%    \item If $\frac{1}{|S|} \sum_{i \in S} \tau_i \leq 24M$, then return $S$.
%    \item Otherwise, pick $T \sim \text{Unif}([0, \max \tau_i])$ and filter out samples above this threshold, i.e. return \[S \leftarrow S \setminus \{i \in S: \tau_i > T\}.\]
%\end{itemize}

This algorithm has two important properties. First, if it does not filter any samples, then the sample mean is provably stable, i.e. it cannot have been affected much by the corruptions, so long as the uncorrupted samples had bounded variance (proof in Appendix~\ref{lemma:filter-correctness-appendix}).

\begin{lemma}\label{lemma:filter-correctness}
Suppose that \textsc{Filter} does not filter out any samples. Then \[\norm{\EE_S \xi - \EE_I \xi}_2 \leq 3\sqrt{48}\sqrt{(M+\norm{\Cov_I(\xi)}_\text{op})\epsilon}\] for any $I \subseteq [n]$ and $\epsilon>0$ such that $|S|,|I| \geq (1-\epsilon)n$.
\end{lemma}

Second, if the threshold is chosen appropriately (based on the variance of the uncorrupted samples), then the filtering step always in expectation removes at least as many corrupted samples as uncorrupted samples. Equivalently, the size of the symmetric difference between the current sample set and the uncorrupted samples (i.e. the number of corrupted samples in the current set plus the number of uncorrupted samples which have been filtered out of the current set) always decreases in expectation (proof in Appendix~\ref{lemma:filter-soundness-appendix}).

\begin{lemma}\label{lemma:filter-soundness}
Consider an execution of \textsc{Filter} with sample set $S$ of size $|S| \geq 2n/3$, and vectors $\{\xi_i: i \in S\}$, and bound $M$. Let $S'$ be the sample set after this iteration's filtering. Let $I_\text{good} \subseteq [n]$ satisfy $|I_\text{good}| \geq (5/6)n$. Suppose that $\Cov_{I_\text{good}}(\xi_i) \preceq MI$, then \[\EE|S' \triangle I_\text{good}| \leq \EE|S \triangle I_\text{good}|,\]
where the expectation is over the random threshold.
\end{lemma}

\section{The \textsc{Iterated-Gmm-Sever} Algorithm}\label{section:gmm-algorithm}

In this section, we describe and analyze an algorithm \textsc{Iterated-Gmm-Sever} for robustly solving moment conditions under Assumption~\ref{assumption:gmm}. The key subroutine is the algorithm \textsc{Gmm-Sever}, which given an initial estimate $w_0$ and a radius $R$ such that the true parameter is contained in $B_R(w_0)$, returns a refined estimate $w$ such that (with large probability) the radius bound can be decreased by a constant factor. 

Like the algorithm \textsc{Sever} \cite{diakonikolas2019sever}, our algorithm \textsc{Gmm-Sever} alternates (a) finding a critical point of a function associated to the current samples, and (b) filtering out ``outlier'' samples. Unlike \textsc{Sever}, the function we optimize is not simply an empirical mean over the samples, but rather the squared-norm of the sample moments. Moreover, we need two filtering steps: the moments as well as directional derivatives of the moments, in a carefully chosen direction. See Algorithm~\ref{alg:gmm-sever} for the complete description.

\begin{algorithm}[t]
    \caption{\textsc{Gmm-Sever}}
    \label{alg:gmm-sever}
    \begin{algorithmic}[1] % The number tells where the line numbering should start
        \Procedure{GMM-Sever}{$\{g_1,\dots,g_n\}, w_0, R, \gamma, \epsilon, L, \sigma$}
            \State $S \gets [n]$
            \Repeat
                \State Compute a $\gamma$-approximate critical point $w \gets \mathcal{L}_{\gamma, B_R(w_0)}(\norm{\EE_S(g_i(\cdot))}_2^2)$
                \State $u \gets \EE_S g_i(w)$
                \State $S' \gets \textsc{Filter}(\{\Gradient g_i(w) \cdot u: i \in S\}, L^2 \norm{u}_2^2)$
                \If{$S' \neq S$}
                    \State $S \gets S'$
                    \State Return to line 4
                \EndIf
                \State $S'' \gets \textsc{Filter}(\{g_i(w): i \in S\}: i \in S\}, \sigma^2 L + 4L^2 R^2)$
                \If{$S'' \neq S$}
                    \State $S \gets S''$
                    \State Return to line 4
                \EndIf
            \Until{$S'' = S$}
            \State \textbf{return} $(w, S)$
        \EndProcedure
    \end{algorithmic}
\end{algorithm}

We will only prove a constant failure probability for \textsc{Gmm-Sever}. However, we will show that it can be amplified to an arbitrarily small failure probability $\delta$. We call the resulting algorithm \textsc{Amplified-Gmm-Sever}; see Algorithm~\ref{alg:amplified-gmm-sever}. The algorithm \textsc{Iterated-Gmm-Sever} then consists of iteratively calling \textsc{Amplified-Gmm-Sever} to refine the parameter estimate and bound the true parameter within successively smaller balls; see Algorithm~\ref{alg:iterated-gmm-sever}.

\begin{algorithm}[t]
    \caption{\textsc{Amplified-Gmm-Sever}}
    \label{alg:amplified-gmm-sever}
    \begin{algorithmic}[1]
        \Procedure{Amplified-GMM-Sever}{$\{g_1,\dots,g_n\}, w_0, R, \gamma, \epsilon, L, \sigma, \delta$}
            \State $t \gets 0$
            \Repeat
                \State $w, S \gets \textsc{Gmm-Sever}(\{g_1,\dots,g_n\}, w_0, R, \gamma, \epsilon, L, \sigma)$
                \State $t \gets t+1$
            \Until{$|S| \geq (1 - 11\epsilon)n$ \textbf{ or } $(1/10)^t \leq \delta$}
            \State \textbf{return} $w$
        \EndProcedure
    \end{algorithmic}
\end{algorithm}

\begin{algorithm}[t]
    \caption{\textsc{Iterated-Gmm-Sever}}
    \label{alg:iterated-gmm-sever}
    \begin{algorithmic}[1]
        \Procedure{Iterated-Gmm-Sever}{$\{g_1,\dots,g_n\}, R_0, \gamma, \epsilon, \lambda, L, \sigma, \delta$}
            \State $t \gets 1$, $w_1 \gets 0$, $R_1 \gets R_0$, $\delta' \gets c\delta/\log(R\sqrt{L}/(\sigma\sqrt{\epsilon})$, $\gamma = \sigma L^{3/2}\sqrt{\epsilon}$
            \Repeat
                \State $\hat{w}_t := \textsc{Amplified-Gmm-Sever}(\{g_1,\dots,g_n\}, w_t, R_t, \epsilon, L, \sigma, \gamma, \delta')$
                \State $R_{t+1} \gets R'_t \gets 2\gamma/\lambda^2 + C((L^2/\lambda^2)R_t\sqrt{\epsilon} + \sigma(L^{3/2}/\lambda^2)\sqrt{\epsilon})$
                \State $t \gets t+1$
            \Until{$R_t > R_{t-1}/2$}
            \State \textbf{return} $\hat{w}_{t-1}$
        \EndProcedure
    \end{algorithmic}
\end{algorithm}

We start by analyzing \textsc{Gmm-Sever}. In the next two lemmas, we show that if the algorithm does not filter out too many samples, then we can bound the distance from the output to $w^*$. First, we show a first-order criticality condition (in the direction $\hat{w} - w^*$) for the norm of the moments of the ``good" samples. If there was no corruption, then we would have an inequality of the form \[\frac{(\hat{w} - w^*)^T}{\norm{\hat{w} - w^*}_2} \EE_{I_\text{good}} \Gradient g(\hat{w})^T \EE_{I_\text{good}} g(\hat{w}) \leq \gamma.\]
With $\epsilon$-corruption, the algorithm is designed so that we can still show the following inequality, matching the above guarantee up to $O(\sqrt{\epsilon})$ (proof in Appendix~\ref{lemma:termination-bound-appendix}):

\begin{lemma}\label{lemma:termination-bound}
Under Assumption~\ref{assumption:gmm}, at algorithm termination, if $|S| \geq (1-10\epsilon)n$, then the output $\hat{w}$ of $\textsc{Gmm-Sever}$ satisfies \[\frac{(\hat{w}-w^*)^T}{\norm{\hat{w}-w^*}_2} \EE_{I_\text{good}} \Gradient g(\hat{w})^T \EE_{I_\text{good}} g(\hat{w}) \leq \gamma + 275\sigma L^{3/2}\sqrt{\epsilon} + 603L^2 R\sqrt{\epsilon}\]
\end{lemma}

Moreover, we can show that any point satisfying the first-order criticality condition must be close to $w^*$, using the least singular value bound on the gradient (proof in Appendix~\ref{lemma:criticality-appendix}).

\begin{lemma}\label{lemma:criticality}
Under Assumption~\ref{assumption:gmm}, suppose that $w \in B_R(w_0)$ satisfies \[(w-w^*)^T \EE_{I_\text{good}} \Gradient g(w)^T \EE_{I_\text{good}} g(w) \leq \kappa \norm{w-w^*}_2.\]
Then $\norm{w-w^*}_2 \leq 4(\kappa + \sigma L^{3/2}\sqrt{\epsilon})/\lambda^2$.
\end{lemma}

Putting the above lemmas together, we immediately get the following bound on $\norm{\hat{w} - w^*}_2$.

\begin{lemma}\label{lemma:termination-bound-final}
Under Assumption~\ref{assumption:gmm}, at algorithm termination, if $|S| \geq (1-10\epsilon)n$, then the output $\hat{w}$ of \textsc{Gmm-Sever} satisfies \[\norm{\hat{w} - w^*}_2 \leq \frac{4\gamma}{\lambda^2} + 2412(L^2/\lambda^2)R\sqrt{\epsilon} + 1102\sigma (L^{3/2}/\lambda^2)\sqrt{\epsilon}.\]
\end{lemma}

It remains to bound the size of $S$ at termination. We follow the super-martingale argument from \cite{diakonikolas2019sever}, which uses Lemma~\ref{lemma:filter-soundness} (proof in Appendix~\ref{theorem:sever-appendix}).

\begin{theorem}\label{theorem:sever}
Suppose that the initial conditions $R$ and $w_0$ satisfy $B_R(w_0) \subseteq B_{2R_0}(0)$. Let $\hat{w}$ be the output of \textsc{Gmm-Sever}. Then with probability at least $9/10$, it holds that \[\norm{\hat{w} - w^*}_2 \leq \frac{4\gamma}{\lambda^2} + 2412(L^2/\lambda^2)R\sqrt{\epsilon} + 1102\sigma (L^{3/2}/\lambda^2)\sqrt{\epsilon}.\]
The time complexity of \textsc{Gmm-Sever} is $O(\poly(n,d,p, T_\gamma))$ where $T_\gamma$ is the time complexity of the $\gamma$-approximate learner $\mathcal{L}$.
Moreover, for any $\delta>0$ the success probability can be amplified to $1-\delta$ by repeating \textsc{Gmm-Sever} $O(\log 1/\delta)$ times, or until $|S| \geq (1-10\epsilon)n$ at termination. We call this \textsc{Amplified-Gmm-Sever}, and it has time complexity $O(\poly(n,d,p, T_\gamma) \cdot \log(1/\delta))$.
\end{theorem}

With the above guarantee for \textsc{Gmm-Sever} and \textsc{Amplified-Gmm-Sever}, we can now analyze \textsc{Iterated-Gmm-Sever} (proof in Appendix~\ref{theorem:gmm-convergence-appendix}).

\begin{theorem}\label{theorem:gmm-convergence}
Suppose that the input to \textsc{Iterated-Gmm-Sever} consists of functions $g_1,\dots,g_n: \RR^d \to \RR$, a corruption parameter $\epsilon>0$, well-conditionedness parameters $\lambda$ and $L$, a Lipschitzness parameter $L_g$, a noise level parameter $\sigma^2$, a radius bound $R_0$, and an optimization error parameter $\gamma$, such that Assumption~\ref{assumption:gmm} is satisfied for some unknown parameter $w^* \in \RR^d$, and $(L^2/\lambda^2)\sqrt{\epsilon} \leq 1/9648$. \footnote{This constant may be improved; we focus in this paper on dependence on the parameters of the problem and do not optimize constants.} Suppose that the algorithm is also given a failure probability parameter $\delta>0$.

Then the output $\hat{w}$ of \textsc{Iterated-Gmm-Sever} satisfies \[\norm{\hat{w} - w^*}_2 \leq O(\sigma(L^{3/2}/\lambda^2)\sqrt{\epsilon})\]
with probability at least $1-\delta$. Moreover, the algorithm has time complexity $O(\poly(n,d,p, T_\gamma) \cdot \log(1/\delta) \cdot \log(R\sqrt{L}/(\sigma\sqrt{\epsilon})))$, where $T_\gamma$ is the time complexity of a $\gamma$-approximate learner and $\gamma = \sigma L^{3/2}\sqrt{\epsilon}$.
\end{theorem}

\section{Applications}

In this section, we apply \textsc{Iterated-Gmm-Sever} to solve linear and logistic instrumental variables regression in the strong contamination model. 

\paragraph{Robust IV Linear Regression}\label{section:iv} Let $Z$ be the vector of $p$ real-valued instruments, and let $X$ be the vector of $d$ real-valued covariates. Suppose that $Z$ and $X$ are mean-zero. Suppose that the response can be described as $Y = X^T w^* + \xi$ for some fixed $w^* \in \RR^d$. The distributional assumptions we will make about $X$, $Y$, and $Z$ are described below.

\begin{assumption}\label{assumption:iv-linear}
Given a corruption parameter $\epsilon > 0$, well-conditionedness parameters $\lambda$ and $L$, hypercontractivity parameter $\tau$, noise level parameter $\sigma^2$, and norm bound $R_0$, we assume the following: (i) \textbf{Valid instruments:} $\EE[\xi|Z] = 0$, (ii) \textbf{Bounded-variance noise:} $\EE[\xi^2|Z] \leq \sigma^2$, (iii) \textbf{Strong instruments:} $\sigma_\text{min}(\EE ZX^T) \geq \lambda$, (iv) \textbf{Boundedness:} $\norm{\Cov([Z; X]}_\text{op} \leq L$, (v) \textbf{Hypercontractivity:} $[Z; X]$ is $(4,2,\tau)$-hypercontractive, (vi) \textbf{Bounded 8th moments:} $\max_i X_i^8 \leq O(\tau^2 L^4)$ and $\max_i Z_i^8 \leq O(\tau^2 L^4)$ (vii) \textbf{Bounded norm parameter:} $\norm{w^*}_2 \leq R_0$.
\end{assumption}

Define \[g(w)= Z(Y - X^T w)\] for $w \in \RR^d$, and let $(X_i,Y_i,Z_i)$ be $n$ independent samples drawn according to $(X,Y,Z)$. Let $\epsilon > 0$. We prove that under the above assumption, if $n$ is sufficiently large, then with high probability, for any $\epsilon$-contamination $(X'_i, Y'_i, Z'_i)_{i=1}^n$ of $(X_i,Y_i, Z_i)_{i=1}^n$, the functions $g_i(w) = Z'_i(Y'_i - (X'_i)^T w)$ satisfy
Assumption~\ref{assumption:gmm}. Formally, we prove the following theorem (see Appendix~\ref{appendix:iv-convergence}):

\begin{theorem}\label{theorem:iv-convergence}
Let $\epsilon > 0$. Suppose that $\epsilon < c\min(\lambda^2/(\tau L^2), \lambda^4/L^4)$ for a sufficiently small constant $c>0$, and suppose that $n \geq C(d+p)^5 \tau \log((p+d)/\tau\epsilon)/\epsilon^2$ for a sufficiently large constant $C$. Then with probability at least $0.95$ over the samples $(X_i, Y_i, Z_i)_{i=1}^n$, the following holds: for any $\epsilon$-corruption of the samples and any upper bound $R_0 \geq \norm{w^*}_2$, Assumption~\ref{assumption:gmm} is satisfied. In that event, if $L$, $\lambda$, $\sigma$, and $\epsilon$ are known, then there is a $\poly(n,d,p,\log(1/\delta),\log(R_0/(\sigma\sqrt{\epsilon})))$-time algorithm which produces an estimate $\hat{w}$ satisfying, with probability at least $1-\delta$: \[\norm{\hat{w} - w^*}_2 \leq O(\sigma (L^{3/2}/\lambda^2)\sqrt{\epsilon})\]
\end{theorem}

\paragraph{Robust IV Logistic Regression}\label{section:logistic-iv} Let $Z$ be a vector of $p$ real-valued instruments, and let $X$ be a vector of $d$ real-valued covariates. Suppose that $Z$ and $X$ are mean-zero. Suppose that the response can be described as $Y = G(X^T w^*) + \xi$ for some fixed $w^* \in \RR^d$, where $G$ is the (unscaled) logistic function. The proofs only use $1$-Lipschitzness of $G$ and $G'$, and that $G'(0)$ is bounded away from $0$.

As far as distributional assumptions, we assume in this section that Assumption~\ref{assumption:iv-linear} holds, and additionally assume that the norm bound satisfies $R_0 \leq c\min(\lambda^2/L, \lambda/\sqrt{\tau L^3})$ for an appropriate constant $c$, where $\lambda$, $L$, and $\tau$ are as required for the Assumption. We obtain the following algorithmic result (proof in Appendix~\ref{appendix:nonlinear-iv-convergence}):

%Let $w^* \in \RR^d$, and suppose that $Y = G(X^T w^*) + \xi$, where $\EE[\xi|Z] = 0$ and $\EE[\xi^2|Z] \leq \sigma^2$. Suppose that $\sigma_\text{min}(\EE ZX^T) \geq \lambda$, and $\norm{\Cov([Z; X])}_\text{op} \leq L$. Suppose that $[Z;X]$ is $\tau$-hypercontractive, and suppose that $\norm{w^*}_2 \leq R_0$ for some known $R_0$ satisfying $R_0 \leq c\min(\lambda^2/L,\lambda/(\tau L^2))$ for an appropriate constant $c$.

\begin{theorem}\label{theorem:nonlinear-iv-convergence}
Let $\epsilon > 0$. Suppose that $\epsilon < c\min(\lambda^2/(\tau L^2), \lambda^4/L^4)$ for a sufficiently small constant $c>0$, and suppose that $n \geq C(d+p)^5 \tau \log((p+d)/\tau\epsilon)/\epsilon^2$ for a sufficiently large constant $C$. Suppose that $\norm{w^*}_2 \leq R_0 \leq c\min(\lambda^2/L,\lambda/\sqrt{\tau L^3})$. Then with probability at least $0.95$ over the samples $(X_i, Y_i, Z_i)_{i=1}^n$, the following holds: for any $\epsilon$-corruption of the samples, Assumption~\ref{assumption:gmm} is satisfied. In that event, if $R_0$, $L$, $\lambda$, $\sigma$, and $\epsilon$ are known, then there is a $\poly(n,d,p,\log(1/\delta),\log(R_0/(\sigma\sqrt{\epsilon})))$-time algorithm which produces an estimate $\hat{w}$ satisfying, with probability at least $1-\delta$: \[\norm{\hat{w} - w^*}_2 \leq O(\sigma (L^{3/2}/\lambda^2)\sqrt{\epsilon})\]
\end{theorem}

\section{Experiments}\label{section:experiments}

In this section we corroborate our theory by applying our algorithm \textsc{Iterated-Gmm-Sever} to several datasets for IV linear regression with heterogeneous treatment effects. This is a natural setting in which the instruments and covariates are high-dimensional, necessitating dimension-independent robust estimators.

\paragraph{IV Linear Regression with Heterogeneous Treatment Effects.} Consider a study in which each sample has a vector $X$ of characteristics, a scalar instrument $Z$, a scalar treatment $T$, and a response $Y$. Assuming that the average treatment effect is linear in the characteristics with unknown coefficients, and that the response noise is independent of the instrument, we can write a moment condition \[\EE[XZ(Y - T\langle X, w^* \rangle - \langle X, \beta \rangle)] = 0.\]
This can be interpreted as an IV linear regression, and therefore our algorithm applies to it.

\paragraph{Synthetic experiment.} For our first experiment, we generate a unknown $20$-dimensional parameter vector $\theta \sim N(0,I)$. We then generate $10000$ independent samples each with a $20$-dimensional characteristic vector $X$ drawn from $N(0, I)$. The instrument $Z$ is drawn from an unbiased Bernoulli distribution, and the binary treatment is drawn from a Bernoulli-$p$ distribution with \[p = \frac{1}{1 + \exp(-Z-\sqrt{20}U\bar{X})},\] where $U$ is a standard normal random variable and $\bar{X} = \frac{1}{20}\sum_{j=1}^{20} X_j$. Finally, the response is \[Y = \langle X, \theta\rangle T + U.\]
The treated samples ($T = 1$) with positive $\bar{X}$ tend to have larger $U$ (and hence larger response), and the treated samples with negative $\bar{X}$ tend to have smaller $U$ (and hence smaller response), so ordinary least-squares would tend to overestimate the treatment effect in the direction of the all-ones vector. However, $U$ is by construction independent of $XZ$, so $Z$ is a valid instrument. Indeed, in the absence of corruption, IV linear regression approximately recovers the true parameter $\theta$.

We then corrupt an $\epsilon$-fraction of the samples, by setting the characteristic vector equal to the all-ones vector (and leaving the instrument and response unchanged). We compute the $\ell_2$ recovery error of  \textsc{Iterated-Gmm-Sever}, classical IV, and two-stage Huber regression for $\epsilon$ varying between $0.01$ and $0.5$. For each choice of $\epsilon$, we repeat the experiment $10$ times and average the recovery errors for each algorithm. See Figure~\ref{fig:synthetic} for the results.

\begin{figure}[h]
    \centering
    \begin{subfigure}{.47\textwidth}
    \centering
    \includegraphics[width=1\textwidth]{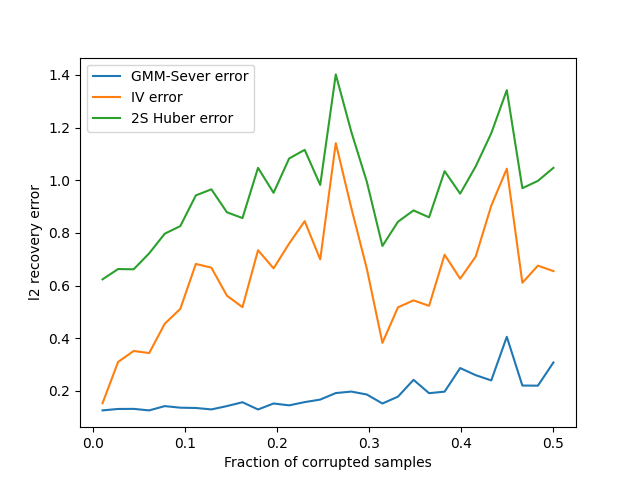}
    \caption{(Synthetic) $\ell_2$ recovery error of \textsc{Iterated-Gmm-Sever} versus baseline classical and robust approaches, for varied levels of corruption.}
    \label{fig:synthetic}
    \end{subfigure}
    ~~~
    \begin{subfigure}{.47\textwidth}
    \centering
    \includegraphics[width=1\textwidth]{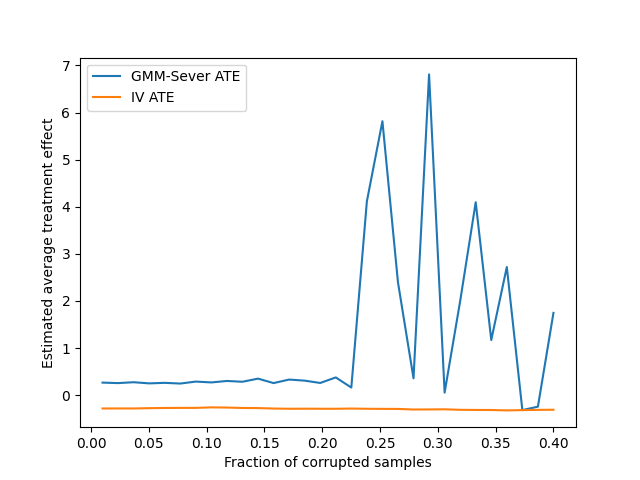}
    \caption{(Semi-synthetic) Average treatment effect (ATE) of \textsc{Iterated-Gmm-Sever} and IV regression for varied levels of corruption.}
    \label{fig:card}
    \end{subfigure}
    \caption{Experimental results for synthetic and semi-synthetic experiment.}
\end{figure}

\vspace{-1em}
\paragraph{Semi-synthetic experiment.} In this experiment, we use the data of Card \cite{card1993using} from the National Longitudinal Survey of Young Men for estimating the average treatment effect (ATE) of education on wages. The data consists of $3010$ samples with years of education as the treatment, log wages as the response, and proximity to a $4$-year college as the instrument, along with $22$ covariates (e.g. geographic and demographic indicator variables).
For simplicity, we restrict the model to only two covariates (years and squared years of labor force experience). We find that the ATE estimated by \textsc{Iterated-Gmm-Sever} is close to the positive ATE estimated by classical IV, suggesting that Card's inference may be robust. Next, we corrupt an $\epsilon$-fraction of the data. Specifically, we solve the IV regression on the uncorrupted data, and we alter the responses of a random $\epsilon$-fraction of the data in a way such that the parameter which satisfies the moment conditions is exactly negated (for $\epsilon n$ not too small, this can be done by solving an underdetermined linear system). This in particular negates the ATE inferred by classical IV regression.

In Figure~\ref{fig:card}, we plot the ATE inferred by IV and the ATE inferred by \textsc{Iterated-Gmm-Sever} as $\epsilon$ is varied from $0.01$ to $0.4$ (the \textsc{Iterated-Gmm-Sever} breaks down on this data for larger levels of corruption). We see that \textsc{Iterated-Gmm-Sever} approximately recovers the correct (positive) ATE of the uncorrupted data, for $\epsilon$ up to around $0.15$. For larger $\epsilon$, the estimate becomes unstable. This is due to the algorithm removing some but not all of the outliers, which are very large norm. We expect that an additional norm thresholding procedure could help remedy this instability, but further investigation may be required. 

\bibliographystyle{plain}
\bibliography{bib}

\begin{thebibliography}{10}

\bibitem{amemiya1982two}
Takeshi Amemiya.
\newblock Two stage least absolute deviations estimators.
\newblock {\em Econometrica: Journal of the Econometric Society}, pages
  689--711, 1982.

\bibitem{bakshi2021robust}
Ainesh Bakshi and Adarsh Prasad.
\newblock Robust linear regression: Optimal rates in polynomial time.
\newblock In {\em Proceedings of the 53rd Annual ACM SIGACT Symposium on Theory
  of Computing}, pages 102--115, 2021.

\bibitem{broderick2021automatic}
Tamara Broderick, Ryan Giordano, and Rachael Meager.
\newblock An automatic finite-sample robustness metric: Can dropping a little
  data change conclusions?, 2021.

\bibitem{card1993using}
David Card.
\newblock Using geographic variation in college proximity to estimate the
  return to schooling, 1993.

\bibitem{diakonikolas2019robust}
Ilias Diakonikolas, Gautam Kamath, Daniel Kane, Jerry Li, Ankur Moitra, and
  Alistair Stewart.
\newblock Robust estimators in high-dimensions without the computational
  intractability.
\newblock {\em SIAM Journal on Computing}, 48(2):742--864, 2019.

\bibitem{diakonikolas2019sever}
Ilias Diakonikolas, Gautam Kamath, Daniel Kane, Jerry Li, Jacob Steinhardt, and
  Alistair Stewart.
\newblock Sever: A robust meta-algorithm for stochastic optimization.
\newblock In {\em International Conference on Machine Learning}, pages
  1596--1606. PMLR, 2019.

\bibitem{diakonikolas2017being}
Ilias Diakonikolas, Gautam Kamath, Daniel~M Kane, Jerry Li, Ankur Moitra, and
  Alistair Stewart.
\newblock Being robust (in high dimensions) can be practical.
\newblock In {\em International Conference on Machine Learning}, pages
  999--1008. PMLR, 2017.

\bibitem{diakonikolas2019efficient}
Ilias Diakonikolas, Weihao Kong, and Alistair Stewart.
\newblock Efficient algorithms and lower bounds for robust linear regression.
\newblock In {\em Proceedings of the Thirtieth Annual ACM-SIAM Symposium on
  Discrete Algorithms}, pages 2745--2754. SIAM, 2019.

\bibitem{durrett2019probability}
Rick Durrett.
\newblock {\em Probability: theory and examples}, volume~49.
\newblock Cambridge university press, 2019.

\bibitem{freue2013natural}
Gabriela V~Cohen Freue, Hernan Ortiz-Molina, and Ruben~H Zamar.
\newblock A natural robustification of the ordinary instrumental variables
  estimator.
\newblock {\em Biometrics}, 69(3):641--650, 2013.

\bibitem{hampel1971general}
Frank~R Hampel.
\newblock A general qualitative definition of robustness.
\newblock {\em The Annals of Mathematical Statistics}, 42(6):1887--1896, 1971.

\bibitem{hampel1974influence}
Frank~R Hampel.
\newblock The influence curve and its role in robust estimation.
\newblock {\em Journal of the american statistical association},
  69(346):383--393, 1974.

\bibitem{hansen1982large}
Lars~Peter Hansen.
\newblock Large sample properties of generalized method of moments estimators.
\newblock {\em Econometrica: Journal of the econometric society}, pages
  1029--1054, 1982.

\bibitem{huber1992robust}
Peter~J Huber.
\newblock Robust estimation of a location parameter.
\newblock In {\em Breakthroughs in statistics}, pages 492--518. Springer, 1992.

\bibitem{huber2004robust}
Peter~J Huber.
\newblock {\em Robust statistics}, volume 523.
\newblock John Wiley \& Sons, 2004.

\bibitem{jambulapati2021robust}
Arun Jambulapati, Jerry Li, Tselil Schramm, and Kevin Tian.
\newblock Robust regression revisited: Acceleration and improved estimation
  rates.
\newblock {\em arXiv preprint arXiv:2106.11938}, 2021.

\bibitem{krasker1986two}
William~S Krasker.
\newblock Two-stage bounded-lnfluence estimators for simultaneous-equations
  models.
\newblock {\em Journal of Business \& Economic Statistics}, 4(4):437--444,
  1986.

\bibitem{ronchetti2001robust}
Elvezio Ronchetti and Fabio Trojani.
\newblock Robust inference with gmm estimators.
\newblock {\em Journal of econometrics}, 101(1):37--69, 2001.

\bibitem{tropp2015introduction}
Joel~A Tropp.
\newblock An introduction to matrix concentration inequalities.
\newblock {\em arXiv preprint arXiv:1501.01571}, 2015.

\end{thebibliography}

\newpage
\appendix

\section{Omitted proofs from Section 3}

\subsection{Proof of Lemma~\ref{lemma:assumption-corollaries}}\label{lemma:assumption-corollaries-appendix}

\begin{proof}
~\paragraph{First claim.} Note that
\begin{align*}
&|\EE_{I_\text{good}}(u^T \Gradient g(w)v)^2 - \EE_{I_\text{good}} (u^T \Gradient g(w^*)v)^2| \\
&\qquad= |\EE_{I_\text{good}}(u^T (\Gradient g(w) - \Gradient g(w^*)) v)(u^T (\Gradient g(w) + \Gradient g(w^*))v)| \\
&\qquad\leq \EE_{I_\text{good}}(u^T (\Gradient g(w) - \Gradient g(w^*)) v)^2 + 2\sqrt{\EE_{I_\text{good}}(u^T (\Gradient g(w) - \Gradient g(w^*)) v)^2 \EE_{I_\text{good}}(u^T \Gradient g(w^*) v)^2} \\
&\qquad\leq L_g^2 \norm{w-w^*}_2^2 + 2L_gL\norm{w-w^*}_2 \\
&\qquad\leq \lambda^2/16 + \lambda L/2 \\
&\qquad\leq L^2
\end{align*}
where the first inequality expands $\Gradient g(w) + \Gradient g(w^*)$ as $(\Gradient g(w) - \Gradient g(w^*)) + 2\Gradient g(w^*)$ and applies Cauchy-Schwarz to the resulting second term; the second inequality applies the Lipschitz gradient assumption and bounded-variance gradient assumption at $w^*$; the third inequality applies the stability of gradient assumption; and the fourth inequality uses that $\lambda \leq L$. It follows that
\[\EE_{I_\text{good}}(u^T \Gradient g(w)v)^2 \leq L^2 + \EE_{I_\text{good}}(u^T \Gradient g(w^*) v)^2 \leq 2L^2\] as claimed.

\paragraph{Second claim.} Observe that for any unit vector $v$, \[\EE_{I_\text{good}}(v\cdot g(w))^2 \leq 2\EE_{I_\text{good}}(v\cdot g(w^*))^2 + 2\EE_{I_\text{good}}(v\cdot(g(w)-g(w^*)))^2.\]
The first term is at most $2\sigma^2 L$ by the bounded-variance noise assumption. The second term can be written and bounded as
\begin{align*}
\EE_{I_\text{good}} (v\cdot (g(w) - g(w^*)))^2
&= \EE_{I_\text{good}} \left(\int_0^1 v^T \Gradient g(tw + (1-t)w^*)(w-w^*) \, dt \right)^2 \\
&\leq \int_0^1 \EE_{I_\text{good}} (v^T \Gradient g(tw + (1-t)w^*)(w-w^*))^2 \\
&\leq 2L^2 \norm{w-w^*}_2^2
\end{align*}
by the first claim. This proves the second claim. 

\paragraph{Third claim.} We have for any $w \in B_{2R_0}(0)$ that $\norm{w-w^*}_2 \leq 4R_0$, so
\begin{align*}
\sigma_\text{min}(\EE_{I_\text{good}}\Gradient g(w)) 
&\geq \sigma_\text{min}(\EE_{I_\text{good}}\Gradient g(w^*)) - \norm{\EE_{I_\text{good}}\Gradient g(w) - \EE_{I_\text{good}} \Gradient g(w^*)}_\text{op} \\
&\geq \lambda - L_g \cdot 4R_0 \\
&\geq \lambda/2
\end{align*}
as claimed, where the second inequality uses the strong identifiability assumption and Lipschitz gradient assumption, and the third inequality uses the stability of gradient assumption.

\paragraph{Fourth claim.} We note that
\[\EE_{I_\text{good}} g(w) - \EE_{I_\text{good}} g(w^*) = \int_0^1 \EE_{I_\text{good}} \Gradient g(tw + (1-t)w^*)(w-w^*) \, dt.\]
The expectation of the gradient has operator norm at most $L + L_g\norm{w-w^*}_2$ by bounded-variance and Lipschitzness of the gradient, and this is at most $2L$ by stability of the gradient and the inequality $\lambda \leq L$. As a result, \[\norm{\EE_{I_\text{good}} g(w) - \EE_{I_\text{good}} g(w^*)}_2 \leq 2L\norm{w-w^*}_2,\] so together with well-specification it follows that $\norm{\EE_{I_\text{good}} g(w)}_2 \leq \sigma\sqrt{L} + 2L\norm{w-w^*}_2$ as claimed.

\paragraph{Fifth claim.} This follows immediately from the first claim. Indeed, for any $w \in B_{2R_0}(0)$ and unit vectors $u \in \RR^p$ and $v \in \RR^d$,
\[(u^T \EE_{I_\text{good}} \Gradient g(w) v)^2 = (\EE_{I_\text{good}} u^T \Gradient g(w) v)^2 \leq \EE_{I_\text{good}} (u^T \Gradient g(w) v)^2 \leq 2L^2.\]
Taking the supremum over all $u,v$ we get that $\norm{\EE_{I_\text{good}}\Gradient g(w)}_\text{op} \leq L\sqrt{2}$ as claimed.
\end{proof}

\section{Omitted proofs from Section~4}

\subsection{Proof of Lemma~\ref{lemma:filter-correctness}}\label{lemma:filter-correctness-appendix}
\begin{proof}
If the algorithm does not remove any samples, then it holds that \[\norm{\Cov_S(\xi)}_\text{op} = \Var_S(v\cdot \xi) = \frac{1}{|S|}\sum_{i \in S} \tau_i \leq 24M.\]
The claim then follows from application of Lemma~\ref{lemma:key} to sets $S$ and $I$, since the total variation distance between the uniform distribution on $S$ and the uniform distribution on $I$ is at most $2\epsilon$.
\end{proof}

\subsubsection{Proof of Lemma~\ref{lemma:filter-soundness}}\label{lemma:filter-soundness-appendix}

\begin{proof}
If no elements are filtered out, then the inequality trivially holds. Suppose otherwise. The difference $|S'\xor I_\text{good}| - |S\xor I_\text{good}|$ is precisely the number of good elements (i.e. $i \in I_\text{good}$) filtered out in this iteration minus the number of bad elements filtered out in this iteration. Due to the random thresholding, the expectation of the former is $\frac{1}{\max \tau_i} \sum_{i \in S\cap I_\text{good}} \tau_i$, and the expectation of the latter is $\frac{1}{\max \tau_i} \sum_{i \in S\setminus I_\text{good}} \tau_i$. Thus, we need to show that $ \sum_{i \in S\cap I_\text{good}} \tau_i \leq  \sum_{i \in S\setminus I_\text{good}} \tau_i$.

Define $S_\text{good} = S\cap I_\text{good}$ and $S_\text{bad} = S \setminus I_\text{good}$. Let $v$ be the largest eigenvector of $\Cov_S(\xi_i)$. We have that
\begin{align*}
\Var_{S_\text{good}}(v\cdot \xi_i)
&= \EE_{S_\text{good}}(v\cdot \xi_i - \EE_{S_\text{good}} v\cdot \xi_i)^2 \\
&\leq \EE_{S_\text{good}}(v\cdot \xi_i - \EE_{I_\text{good}} v\cdot \xi_i)^2 \\
&\leq 2\EE_{I_\text{good}}(v\cdot \xi_i - \EE_{I_\text{good}} v\cdot \xi_i)^2 \\
&= 2\Var_{I_\text{good}}(v\cdot \xi_i) \\
&\leq 2M
\end{align*}
where the first inequality uses the fact that variance is the smallest second moment obtainable by shifting; the second inequality uses that $|S_\text{good}| \geq (2/3 - 1/6)n \geq |I_\text{good}|/2$ and $S_\text{good} \subseteq I_\text{good}$; and the third inequality is by the lemma's assumption.

On the other hand, since the algorithm doesn't terminate, it holds that \[\Var_S(v\cdot \xi_i) = \frac{1}{|S|} \sum_{i\in S} \tau_i \geq 24M.\]

Defining $\mu_\text{good} = \EE_{S_\text{good}} v\cdot \xi_i$, $\mu_\text{bad} = \EE_{S_\text{bad}} v\cdot \xi_i$, and $\mu = \EE_S v\cdot \xi_i$, it follows that \[\frac{1}{|S_\text{good}|} \sum_{i \in S_\text{good}} \tau_i = \EE_{S_\text{good}}(v\cdot \xi_i - \mu)^2 \leq 2M + (\mu - \mu_\text{good})^2.\]
There are two cases to consider:
\begin{enumerate}
    \item If $(\mu-\mu_\text{good})^2 \leq 8M$. Then \[\frac{1}{|S_\text{good}|} \sum_{i \in S_\text{good}} \tau_i \leq 12M \leq \frac{1}{2} \Var_S(v\cdot \xi_i) = \frac{1}{2|S|} \sum_{i \in S} \tau_i.\]
    Thus, \[\sum_{i \in S_\text{good}} \tau_i \leq \frac{1}{2} \sum_{i \in S} \tau_i \leq \sum_{i \in S_\text{bad}} \tau_i.\]
    \item If $(\mu - \mu_\text{good})^2 \geq 8M$. By the above calculation, \[\frac{1}{|S_\text{good}|}\sum_{i \in S_\text{good}} \tau_i \leq 1.5(\mu-\mu_\text{good})^2.\]
    On the other hand, \[\frac{1}{|S_\text{bad}|}\sum_{i\in S_\text{bad}}\tau_i = \EE_{S_\text{bad}}(v\cdot \xi_i - \mu)^2 \geq (\mu - \mu_\text{bad})^2.\]
    But $|\mu-\mu_\text{good}|\cdot |S_\text{good}| = |\mu - \mu_\text{bad}| \cdot |S_\text{bad}|$. As a result, \[\sum_{i \in S_\text{good}} \tau_i \leq 1.5 |S_\text{good}| \cdot (\mu - \mu_\text{good})^2 = 1.5 \frac{|S_\text{bad}|^2}{|S_\text{good}|} (\mu-\mu_\text{bad})^2 \leq 1.5\frac{|S_\text{bad}|}{|S_\text{good}|} \sum_{i \in S_\text{bad}} \tau_i.\]
    But $1.5|S_\text{bad}|/|S_\text{good}| \leq 1.5(n/6)/(2n/3 - n/6) \leq 1$.
\end{enumerate}
In either case, the desired claim holds.
\end{proof}

\section{Omitted proofs from Section~5}

\subsection{Proof of Lemma~\ref{lemma:termination-bound}}\label{lemma:termination-bound-appendix}

\begin{proof}
By the termination conditions of \textsc{Gmm-Sever}, no samples are filtered out in the last iteration. Thus, by Lemma~\ref{lemma:filter-correctness} and the bounds $|S|,|I| \geq (1-10\epsilon n)$, since no samples are filtered out on Step 3, it holds that
\begin{align*}
\norm{\EE_S \Gradient g(\hat{w})^T u - \EE_{I_\text{good}} \Gradient g(\hat{w})^T u}_2
&\leq 3\sqrt{48}\sqrt{(L^2\norm{u}_2^2 + \norm{\Cov_{I_\text{good}}(\Gradient g(\hat{w})^T u)}_\text{op}) \cdot 10\epsilon} \\
&\leq 36\sqrt{10}L\norm{u}_2\sqrt{\epsilon}
\end{align*}
where the last inequality uses the guarantee of Lemma~\ref{lemma:assumption-corollaries} that $\EE_{I_\text{good}}(u^T\Gradient g(\hat{w})v)^2 \leq 2L^2$ for unit vectors $u,v$.

In the second filter operation, since no samples are filtered out, Lemma~\ref{lemma:filter-correctness} implies that 
\begin{align*}
\norm{\EE_S g(\hat{w}) - \EE_{I_\text{good}} g(\hat{w})}_2
&\leq 3\sqrt{48}\sqrt{(\sigma^2L + 4L^2R^2 + \norm{\Cov_{I_\text{good}}(g(\hat{w}))}_\text{op}) \cdot 10\epsilon} \\
&\leq 36\sqrt{10} \sigma\sqrt{L\epsilon} + 120\sqrt{6} LR\sqrt{\epsilon}
\end{align*}
where the last inequality uses that $\Cov_{I_\text{good}}(g_i(\hat{w})) \preceq 2\sigma^2 L + 16L^2 R^2$ by Lemma~\ref{lemma:assumption-corollaries}. Next, since $\norm{\EE_{I_\text{good}} \Gradient g(\hat{w})}_\text{op} \leq \sqrt{2}L$ by Lemma~\ref{lemma:assumption-corollaries}, it follows that \[\norm{\EE_{I_\text{good}} \Gradient g(\hat{w})^T (\EE_S g(\hat{w}) - \EE_{I_\text{good}} g(\hat{w}))}_2 \leq 72\sqrt{5}\sigma L^{3/2}\sqrt{\epsilon} + 240\sqrt{3}L^2 R\sqrt{\epsilon}.\]
Together with the first inequality, we get that
\[\norm{\EE_S\Gradient g(\hat{w})^T \EE_S g(\hat{w}) - \EE_{I_\text{good}} \Gradient g(\hat{w}) \EE_{I_\text{good}} g(\hat{w})}_2 \leq 72\sqrt{5}\sigma L^{3/2}\sqrt{\epsilon} + 240\sqrt{3}L^2 R\sqrt{\epsilon} + 36\sqrt{10}L\norm{u}_2\sqrt{\epsilon}.\]
By assumption, $\norm{\EE_{I_\text{good}} g(\hat{w})}_2 \leq \sigma\sqrt{L\epsilon} + L\norm{\hat{w}-w^*}_2 \leq \sigma\sqrt{L\epsilon} + 2LR$. Therefore $\norm{u}_2 = \norm{\EE_S g(\hat{w})}_2 \leq \sigma\sqrt{L} + 3LR$ assuming that $\max(36\sqrt{10},120\sqrt{6})\sqrt{\epsilon} \leq 1$. Substituting this bound, we get \[\norm{\EE_S\Gradient g(\hat{w})^T \EE_S g(\hat{w}) - \EE_{I_\text{good}} \Gradient g(\hat{w}) \EE_{I_\text{good}} g(\hat{w})}_2 \leq (72\sqrt{5}+36\sqrt{10})\sigma L^{3/2}\sqrt{\epsilon} + (240\sqrt{3}+108\sqrt{3})L^2 R\sqrt{\epsilon}.\]
Now recall that $\hat{w}$ is a $\gamma$-critical point of $\norm{\EE_S g(w)}_2^2$ in the region $B_R(w_0)$. Since $w^* \in B_R(w_0)$, the line segment between $\hat{w}$ and $w^*$ is also contained in $B_R(w_0)$, so by definition of a $\gamma$-critical point, it holds that \[(w^* - \hat{w}) \cdot \EE_S \Gradient g(\hat{w})^T \EE_S g(\hat{w}) \geq -\gamma \norm{\hat{w} - w^*}_2.\]
So by the triangle inequality, and rounding up the above constants to integers,
\[(\hat{w}-w^*)^T \EE_{I_\text{good}} \Gradient g(\hat{w})^T \EE_{I_\text{good}} g(\hat{w}) \leq \gamma \norm{\hat{w}-w^*}_2 + (275\sigma L^{3/2}\sqrt{\epsilon} + 603L^2 R\sqrt{\epsilon}) \norm{\hat{w}-w^*}_2\]
as claimed.
\end{proof}

\subsection{Proof of Lemma~\ref{lemma:criticality}}\label{lemma:criticality-appendix}

\begin{proof}
Expanding $\EE_{I_\text{good}} g(w) - g(w^*)$ as an integral, we have that
\begin{align*}
&(w-w^*)^T \EE_{I_\text{good}} \Gradient g(w)^T \EE_{I_\text{good}} (g(w) - g(w^*)) \\
&= (w-w^*)^T \EE_{I_\text{good}} \Gradient g(w)^T \int_0^1 \EE_{I_\text{good}} \Gradient g(tw + (1-t)w^*)(w-w^*) \, dt \\
&= (w-w^*)^T \EE_{I_\text{good}} \Gradient g(w)^T \int_0^1 \EE \Gradient g(w)(w-w^*) \, dt \\
&\qquad + (w-w^*)^T \EE_{I_\text{good}} \Gradient g(w)^T \int_0^1 (\EE_{I_\text{good}} \Gradient g(tw + (1-t)w^*) - \EE_{I_\text{good}} \Gradient g(w)) (w-w^*) \, dt.
\end{align*}
Now, the first term is precisely $\norm{\EE_{I_\text{good}} \Gradient g(w) (w-w^*)}_2^2$. We bound the absolute value of the second term by Cauchy-Schwarz and the Lipschitzness of the gradient; it is at most \[\norm{\EE_{I_\text{good}} \Gradient g(w) (w-w^*)}_2 \cdot L_g \norm{w-w^*}_2^2.\]
As a result,
\begin{align*}
(w-w^*)^T \EE_{I_\text{good}} \Gradient g(w)^T \EE_{I_\text{good}} (g(w) - g(w^*)) 
&\geq \norm{\EE_{I_\text{good}} \Gradient g(w)(w-w^*)}_2^2 \\
&\qquad- L_g\norm{w-w^*}_2^2 \norm{\EE_{I_\text{good}} \Gradient g(w)(w-w^*)}_2.
\end{align*}
Suppose that $\norm{\EE_{I_\text{good}} \Gradient g(w) (w-w^*)}_2 \leq 2L_g\norm{w-w^*}_2^2$. Then by assumption that $\sigma_\text{min}(\EE_{I_\text{good}} \Gradient g(w)) \geq \lambda$, it follows that $\norm{w-w^*}_2 \leq (2L_g/\lambda)\norm{w-w^*}_2^2$. Thus $\norm{w-w^*}_2 \geq \lambda/(2L_g)$, which contradicts the assumptions that $R_0 < \lambda/(4L_g)$ and $w \in B_R(w_0) \subseteq B_{2R_0}(0)$. We conclude that in fact $\norm{\EE_{I_\text{good}} \Gradient g(w)(w-w^*)}_2 > 2L_g \norm{w-w^*}_2^2$, so that
\[(w-w^*)^T \EE_{I_\text{good}} \Gradient g(w)^T \EE_{I_\text{good}} (g(w)-g(w^*)) \geq \frac{1}{2}\norm{\EE_{I_\text{good}} \Gradient g(w)(w-w^*)}_2^2.\]
However, by Assumption~\ref{assumption:gmm} and Lemma~\ref{lemma:assumption-corollaries}, \[|(w-w^*)^T \EE_{I_\text{good}} \Gradient g(w)^T \EE_{I_\text{good}} g(w^*)| \leq \sqrt{2}\sigma L^{3/2} \norm{w-w^*}_2 \sqrt{\epsilon}.\]
So together with the lemma's assumption, \[(w-w^*)^T \EE_{I_\text{good}} \Gradient g(w)^T \EE_{I_\text{good}} (g(w)-g(w^*)) \leq (\kappa + \sqrt{2} \sigma L^{3/2}\sqrt{\epsilon}) \norm{w-w^*}_2.\]
As a result, \[\norm{\EE_{I_\text{good}} \Gradient g(w)(w-w^*)}_2^2 \leq 2(\kappa+\sqrt{2}\sigma L^{3/2} \sqrt{\epsilon}) \norm{w-w^*}_2,\]
so that by the least singular value bound in Lemma~\ref{lemma:assumption-corollaries}, $\norm{w-w^*}_2 \leq 4(\kappa+\sqrt{2}\sigma L^{3/2}\sqrt{\epsilon})/\lambda^2$ as claimed.
\end{proof}

\subsection{Proof of Theorem~\ref{theorem:sever}}\label{theorem:sever-appendix}

\begin{proof}
For $t\geq 1$ let $S_t$ be the algorithm's sample set at the beginning of the $t$-th iteration, so that $S_1 = [n]$. Define a ``sticky" stochastic process based on $|S_t \xor I_\text{good}|$: \[X_t = \begin{cases} |S_t \xor I_\text{good}| & \text{ if } |S_{t-1}| \geq 2n/3 \\ X_{t-1} & \text{ otherwise}\end{cases}.\]
By soundness of the filtering algorithm (Lemma~\ref{lemma:filter-soundness}), we know that $(X_t)_{t\geq 1}$ is a super-martingale. By Ville's maximal inequality \cite{durrett2019probability} and since $\EE X_1 = \epsilon n$, it holds with probability at least $8/9$ that $\sup_t X_t \leq 9\epsilon n$. In this event, $|S_t| \geq 2n/3$ for all $t$, so $\sup_t |S_t \xor I_\text{good}| \leq 9\epsilon n$, and therefore $\inf_t |S_t| \geq (1-10\epsilon)n$. In particular, $|S| \geq (1-10\epsilon)n$, where $S$ is the terminal sample set. Then by Lemma~\ref{lemma:termination-bound-final}, it follows that \[\norm{\hat{w} - w^*}_2 \leq \frac{4\gamma}{\lambda^2} + 2412(L^2/\lambda^2)R\sqrt{\epsilon} + 1102\sigma (L^{3/2}/\lambda^2)\sqrt{\epsilon}.\]
where $\hat{w}$ is the output of GMM-Sever. Since this bound holds deterministically whenever $|S| \geq (1-10\epsilon)n$, the failure probability can be decreased to $\delta$ by repeating GMM-Sever until either $|S| \geq (1-10\epsilon)n$, or $O(\log 1/\delta)$ repetitions have occurred.

The time complexity bound follows from observing that the \textsc{Filter} algorithm runs in polynomial time, and in each repetition at least one sample is removed from $S$, so the algorithm terminates after at most $n$ repetitions.
\end{proof}

\subsection{Proof of Theorem~\ref{theorem:gmm-convergence}}\label{theorem:gmm-convergence-appendix}

\begin{proof}
Formally, \textsc{Iterated-Gmm-Sever} does the following procedure:
\begin{enumerate}
    \item Initialize $t=1$, $w_1 = 0$, $R_1 = R_0$, $\delta' = c\delta/\log(R\sqrt{L}/(\sigma\sqrt{\epsilon}))$, and $\gamma = \sigma L^{3/2}\sqrt{\epsilon}$
    \item Compute $\hat{w}_t := \textsc{Amplified-Gmm-Sever}(\{g_1,\dots,g_n\}, w_t, R_t, \epsilon, \lambda, L, \sigma, \gamma, \delta')$
    \item Set $R'_t := 4\gamma/\lambda^2 + 2412((L^2/\lambda^2)R_t\sqrt{\epsilon} + \sigma (L^{3/2}/\lambda^2)\sqrt{\epsilon})$
    \item If $R'_t > R_t/2$, then terminate and return $\hat{w}_t$. Otherwise, set $w_{t+1} := \hat{w}_t$, $R_{t+1} := R'_t$, and return to step (2).
\end{enumerate}

First, note that by induction and the termination condition, $R$ is halved in every iteration, so it holds for all $t \geq 1$ that $R_t \leq R_0 / 2^{t-1}$.

\paragraph{Runtime.} The termination condition is deterministic. In particular, the algorithm will terminate once \[\frac{4\gamma}{\lambda^2} + 2412((L^2/\lambda^2) R_t \sqrt{\epsilon} + \sigma (L^{3/2}/\lambda^2)\sqrt{\epsilon}) > R_t/2.\]
This holds if $R_t < 4824\sigma(L^{3/2}/\lambda^2)\sqrt{\epsilon}$. Since $R_t$ halves in every iteration and $\lambda \leq L$, the algorithm will therefore terminate after at most $O(\log(R\sqrt{L}/(\sigma\sqrt{\epsilon})))$ iterations. By the runtime bound on \textsc{Amplified-Gmm-Sever}, it follows that \textsc{Iterated-Gmm-Sever} has time complexity $O(\poly(n,d,p, T_\gamma) \cdot \log(1/\delta) \cdot \log(R\sqrt{L}/(\sigma\sqrt{\epsilon}))$.

\paragraph{Correctness.} Next, we claim by induction that after the $t$-th call to \textsc{Amplified-Gmm-Sever}, it holds with probability at least $1 - \delta' t$ that $\norm{\hat{w}_t - w^*}_2 \leq R'_t$. For $t=1$, this follows from Theorem~\ref{theorem:sever} and the assumption that $w^* \in B_{R_0}(0)$ (which implies that $w^* \in B_{R_1}(w_1) \subseteq B_{2R_0}(0)$).

Now fix any $t > 1$ for which the algorithm has not yet terminated, and condition on $\norm{\hat{w}_{t-1} - w^*}_2 \leq R'_{t-1}$. Then by the triangle inequality, $$\norm{w_t}_2 = \norm{\hat{w}_{t-1}}_2 \leq \norm{w^*}_2 + R'_{t-1} \leq R_0 + R_0/2^{t-1}.$$
As a result, $B_{R_t}(w_t) \subseteq B_{R_0 + 2R_0/2^{t-1}}(0) \subseteq B_{2R_0}(0)$. In this event, by Theorem~\ref{theorem:sever}, it holds with probability at least $1-\delta'$ that $\norm{\hat{w}_t - w^*}_2 \leq R'_t$. By the induction hypothesis, the event we conditioned on occurs with probability at least $1-\delta'(t-1)$, so by a union bound, it holds that $\norm{\hat{w}_t - w^*}_2 \leq R'_t$ with probability at least $1-\delta't$, completing the induction.

Now consider the final iteration $t$. Restating the termination condition, we have \[\frac{4\gamma}{\lambda^2} + 2412((L^2/\lambda^2) R_t \sqrt{\epsilon} + \sigma (L^{3/2}/\lambda^2)\sqrt{\epsilon}) > R_t/2.\]
By assumption that $(L^2/\lambda^2)\sqrt{\epsilon} \leq 1/9648$, it follows that \[R_t \leq \frac{16\gamma}{\lambda^2} + 9648\sigma (L^{3/2}/\lambda^2)\sqrt{\epsilon}.\]
Thus, with probability at least $1-\delta't$, the output $\hat{w}_t$ of \textsc{Iterated-Gmm-Sever} satisfies
\begin{align*}
\norm{\hat{w}_t - w^*}_2 
&\leq R'_t \\
&\leq \frac{4\gamma}{\lambda^2} + \frac{R_t}{4} + 2412\sigma (L^{3/2}/\lambda^2)\sqrt{\epsilon} \\
&\leq \frac{4\gamma}{\lambda^2} + 4824\sigma(L^{3/2}/\lambda^2)\sqrt{\epsilon}.
\end{align*}
By the choice of $\gamma$, this bound is $O(\sigma(L^{3/2}/\lambda^2)\sqrt{\epsilon})$. By the iteration bound and choice of $\delta'$, the overall failure probability is at most $\delta$.
%We initialize $w_0 = 0$ and $R = R_0$. We then call \textsc{Gmm-Sever} (amplified to failure probability $1/O(\log(R\sqrt{L}/\sigma))$) repeatedly \vscomment{Shouldn't $\delta$ also enter the latter quantity?}. In call $i$, we set $w_0$ to be the output $\hat{w}_{i-1}$ of call $i-1$, and set $R = R_0/2^{i-1}$, and $\gamma = \lambda^2 R\sqrt{\epsilon}/8$. Inductively, we have that $\norm{\hat{w}_i - \hat{w}_{i-1}}_2 \leq R_0/2^i$, so that \[\norm{\hat{w}_i}_2 \leq \sum_{j=1}^i \norm{\hat{w}_j - \hat{w}_{j-1}}_2 \leq R_0.\] Thus, for every call $i$ to \textsc{Gmm-Sever} it holds that $B_R(\hat{w}_{i-1}) \subseteq B_{R_0}(\hat{w}_{i-1}) \subseteq B_{2R_0}(0)$. Moreover $\norm{w^* - \hat{w}_{i-1}}_2 \leq R_0/2^{i-1}$ by the induction hypothesis. So the assumptions of \textsc{Gmm-Sever} are met in call $i$\vscomment{Could you elaborate which assumptions it is that you are trying to meet?}, completing the induction. \vscomment{I totally don't understand how you completed the induction??? How did you show that $\|\hat{w}_{i+1}-\hat{w}_i\|\leq R_0/ 2^{i+1}$?? In particular, I'm also confused because the statements of the guarantee at each iteration of GMM sever, by the preceding theorems does not show that the dependence on $R$ is halved at each iteration. Can you please elaborate much more on this inductive proof.} We iterate until the bound from the above theorem exceeds $R/2$, at which point it holds that $R \leq O(\sigma (L^{3/2}/\lambda^2)\sqrt{\epsilon})$. So the same bound holds for $\norm{\hat{w} - w^*}_2$ as desired.
\end{proof}

\section{Proof of Theorem~\ref{theorem:iv-convergence}}\label{appendix:iv-convergence}

We need to prove that the contaminated samples $(X'_i,Y'_i,Z'_i)_{i=1}^n$ satisfy Assumption~\ref{assumption:gmm} with some set $I_\text{good}$ of size $(1-O(\epsilon))n$. To this end, it suffices to prove that with high probability over the original samples $(X_i,Y_i,Z_i)_{i=1}^n$, there is a subset $I$ of these original samples, with $|I| \geq (1-\epsilon)n$, such that for any subset $S \subseteq I$ of size at least $(1-2\epsilon)n$, the conditions of Assumption~\ref{assumption:gmm} are satisfied. In this event, the intersection of $I$ with the uncontaminated samples certifies the assumption.

In the subsequent lemmas, we verify one by one that each condition of Assumption~\ref{assumption:gmm} is satisfied with high probability for all subsets $S$ of size at least $(1-\epsilon)n$ of a set $I$ of size at least $(1-\epsilon/100)n)$; we then take the intersection of the sets $I$ to yield a set $I_\text{good}$ witnessing Assumption~\ref{assumption:gmm}.

\begin{lemma}\label{lemma:linear-least-singular}
Let $\epsilon>0$ be sufficiently small. If $n \geq C(d+p)^3\sqrt{\tau}\log(1/\tau\epsilon)/\epsilon^2$ then with probability at least $0.99$, there is a subset $I \subseteq [n]$ of size $|I| \geq (1-\epsilon/100)n$ such that for every subset $S \subseteq I$ with $|S| \geq (1-\epsilon)n$, it holds that \[\norm{\frac{1}{|S|} \sum_{i\in S} \begin{bmatrix} Z \\ X\end{bmatrix} \begin{bmatrix} Z^T & X^T \end{bmatrix} - \EE \begin{bmatrix} Z \\ X\end{bmatrix} \begin{bmatrix} Z^T & X^T \end{bmatrix}}_\text{op} \leq O(\sqrt{\tau\epsilon}L).\]
As a consequence, if $(L/\lambda)\sqrt{\tau\epsilon}$ is less than a sufficiently small constant, then \[\sigma_\text{min}\left(\frac{1}{|S|} \sum_{i \in S} Z_i X_i^T \right) \geq \lambda/2.\]
\end{lemma}

\begin{proof}
The first statement follows from Corollary~\ref{corollary:empirical-covariance}, $\tau$-hypercontractivity of $[Z;X]$, and the covariance upper bound on $[Z;X]$. Let $\hat{M} = \frac{1}{|S|}\sum_{i\in S} Z_iX_i^T$. It follows from the first statement, that for any $u$, \[\norm{\hat{M}u - \EE ZX^T u}_2 \leq O(L\sqrt{\tau\epsilon}) \norm{u}_2 \leq \frac{\lambda}{2}\norm{u}_2.\]
By assumption that $\sigma_\text{min}(\EE ZX^T) \geq \lambda$, it follows that $\norm{\hat{M}u}_2 \geq (\lambda/2)\norm{u}_2$. The second statement follows.
\end{proof}

\begin{lemma}\label{lemma:linear-gradient-upper}
Let $\epsilon>0$ and suppose that $n \geq C(p+d)^5\log((p+d)/\epsilon)/\epsilon^2$ for an appropriate constant $C$. Then with probability $0.98$, there is a set $I \subseteq [n]$ with $|I| \leq (1-\epsilon)n$ such that $\EE_I (u^T Z)^2(v^T X)^2 \preceq C\tau L^2 I$ for all unit vectors $u \in \RR^p$ and $v \in \RR^d$.
\end{lemma}

\begin{proof}
By hypercontractivity, we have \[\EE\langle X,u\rangle^4 \leq \tau \left(\EE\langle X,u\rangle^2\right)^2 \leq \tau L^2 \norm{u}_2^4\] for any vector $u \in \RR^d$, and similarly for $Z$. Moreover, we have assumed that the coordinates of $X$ and $Z$ have $8$th moments bounded by $O(\tau^2 L^4)$. Thus, we can apply Lemma~\ref{lemma:fourth-moment-concentration} to $X/\sqrt[4]{\tau L^2}$ and $Z/\sqrt[4]{\tau L^2}$ to get sets $I_1, I_2 \subseteq [n]$ each of size at least $(1-\epsilon/2)n$, that with probability $0.98$ satisfy \[\frac{1}{|I_1|} \sum_{i \in I_1} \langle X_i, u\rangle^4 \leq C\tau L^2\] and \[\frac{1}{|I_2|} \sum_{i \in I_2} \langle Z_i, v\rangle^4 \leq C\tau L^2\] for all unit vectors $u \in \RR^d$ and $v \in \RR^p$. Let $I = |I_1 \cap I_2|$. Then $|I| \geq (1-\epsilon)n$, and the above bounds hold over $I$ as well up to a constant factor loss. Thus, 
\[\EE_I (u^T Z)^2 (X^T v)^2 \leq \sqrt{\EE_I\langle Z, u\rangle^4 \EE_I \langle X, v\rangle^4} \leq C\tau L^2.\]
The lemma follows.
\end{proof}

\begin{lemma}\label{lemma:linear-true-variance}
Let $\epsilon>0$ and suppose that $n \geq C(p+d)^3/\epsilon^2$. Then with probability $0.99$, there is a set $I \subseteq [n]$ with $|I| \geq (1-\epsilon)n$ such that $\EE_I (v^T Z\xi)^2 \leq C\sigma^2 L$ for every unit vector $v \in \RR^p$.
\end{lemma}

\begin{proof}
Since $\EE[\xi^2|Z] \leq \sigma^2$, observe that $\EE (v^T Z\xi)^2 \leq \sigma^2 \EE (v^T Z)^2 \leq \sigma^2 L$ for every unit vector $v \in \RR^p$. The claim follows from Corollary~\ref{corollary:empirical-covariance}.
\end{proof}

\begin{lemma}\label{lemma:linear-true-norm}
Let $\epsilon > 0$, and suppose that $n \geq C(p^{3/2}/\epsilon)\log(p)$. With probability $0.99$, there is a subset $I \subseteq [n]$ with $|I| \geq (1-\epsilon/100)n$ such that for every $S \subseteq I$ with $|S| \geq (1-\epsilon)n$, it holds that \[\norm{\EE_S Z\xi}_2 \leq O(\sigma \sqrt{L\epsilon}).\]
\end{lemma}

\begin{proof}
Observe that $\EE Z\xi = 0$ and $\EE ZZ^T \xi^2 \preceq \sigma^2 L I$ by assumption. The claim follows from Lemma~\ref{lemma:mean-zero-concentration}.
%By similar arguments as in Corollary~\ref{corollary:empirical-covariance}, we can get a subset $I \subseteq [n]$ of size $|I| \geq (1-\epsilon/2)n$ satisfying $\Cov_I(Z\xi) \preceq 2\sigma^2 LI$, and $\norm{\EE_I Z\xi}_2 \leq O(\sigma\sqrt{L\epsilon})$. By Lemma~\ref{lemma:key} and the covariance bound it follows that for any $|S| \subseteq I$ with $|S| \geq (1-\epsilon)n$, it also holds that $\norm{\EE_S Z\xi}_2 \leq O(\sigma\sqrt{L\epsilon})$. See Lemma~B.2 in \cite{cherapanamjeri2020optimal} for more details.
\end{proof}

\begin{corollary}
Let $\epsilon>0$. Suppose that $\epsilon < c\min(\lambda^2/(\tau L^2), \lambda^4/L^4)$ for a sufficiently small constant $c>0$, and suppose that $n \geq C(d+p)^5 \tau \log((p+d)/\tau\epsilon)/\epsilon^2$ for a sufficiently large constant $C$. Then with probability at least $0.95$, there is a set $I \subseteq [n]$ with $|I| \geq (1-\epsilon/2)n$ such that for every subset $S \subseteq I$ with $|S| \geq (1-\epsilon)n$, the following hold:
\begin{itemize}
    \item $\sigma_\text{min}(\EE_S \Gradient g(w^*)) \geq \Omega(\lambda)$
    \item $\EE_S (u^T \Gradient g(w) v)^2 \leq O(\tau L^2)$ for all unit vectors $u,v$ and all $w$
    \item $\EE_S (v^T g(w^*))^2 \leq O(\sigma^2 L)$ for all unit vectors $v$
    \item $\norm{\EE_S g(w^*)}_2 \leq O(\sigma \sqrt{L\epsilon})$
    \item $\Gradient g(w)$ is constant in $w$.
\end{itemize}
\end{corollary}

\begin{proof}
Let $I_1,I_2,I_3,I_4 \subseteq [n]$ be the sets guaranteed by Lemma~\ref{lemma:linear-least-singular} (with parameter $\epsilon$), Lemma~\ref{lemma:linear-gradient-upper} (with parameter $\epsilon/100$), Lemma~\ref{lemma:linear-true-variance} (with parameter $\epsilon/100$), and~\ref{lemma:linear-true-norm} (with parameter $\epsilon$), which satisfy the claims of the respective lemmas with probability at least $0.95$. Let $I = I_1 \cap I_2 \cap I_3 \cap I_4$. We have that $|I_1|,|I_2|,|I_3|,|I_4| \geq (1-\epsilon/100)n$, so $I$ is a subset of each of $I_1,I_2,I_3,I_4$ of size at least $(1-\epsilon/2)n$. Let $S \subseteq I$ have $|S| \geq (1-\epsilon)n$. By Lemma~\ref{lemma:linear-least-singular} and since $S \subseteq I_1$, it holds that $\sigma_\text{min}(\EE_S \Gradient g(w^*)) \geq \lambda/2$. By Lemma~\ref{lemma:linear-gradient-upper} and since $S \subseteq I_2$, it holds that $\EE_S (u^T \Gradient g(w) v)^2 \leq 2\EE_{I_2} (u^T \Gradient g(w) v)^2 \leq O(\tau L^2)$. By Lemma~\ref{lemma:linear-true-variance} we have $\EE_S (v^T g(w^*))^2 \leq O(\sigma^2 L)$, and by Lemma~\ref{lemma:linear-true-norm} we have $\norm{\EE_S g(w^*)}_2 \leq o(\sigma\sqrt{L\epsilon})$. Finally, $\Gradient g(w) = ZX^T$ is clearly constant in $w$.
\end{proof}

The above corollary validates Assumption~\ref{assumption:gmm} for linear instrumental variables. Since $\Gradient g(w)$ is constant in $w$, the Assumption holds for any bound $R_0$ on the norm of the true solution $w^*$. Formally, we can instantiate Theorem~\ref{theorem:gmm-convergence} to get a provably robust estimator for instrumental variables linear regression, as stated in Theorem~\ref{theorem:iv-convergence}.

% \begin{theorem}\label{theorem:iv-convergence-appendix}
% Let $\epsilon > 0$. Suppose that $\epsilon < c\min(\lambda^2/(\tau L^2), \lambda^4/L^4)$ for a sufficiently small constant $c>0$, and suppose that $n \geq C(d+p)^3 \tau \log((p+d)/\tau\epsilon)/\epsilon^2$ for a sufficiently large constant $C$. Then with probability at least $0.95$ over the samples $(X_i, Y_i, Z_i)_{i=1}^n$, the following holds: for any $\epsilon$-corruption of the samples and any upper bound $R_0 \geq \norm{w^*}_2$, Assumption~\ref{assumption:gmm} is satisfied. In that event, if $L$, $\lambda$, $\sigma$, and $\epsilon$ are known, then there is a $\poly(n,d,p,\log(1/\delta),\log(R_0/(\sigma\sqrt{\epsilon})))$-time algorithm which produces an estimate $\hat{w}$ satisfying \[\norm{\hat{w} - w^*}_2 \leq O(\sigma (L^{3/2}/\lambda^2)\sqrt{\epsilon})\] with probability at least $1-\delta$.
% \end{theorem}

\begin{remark}
Although Theorem~\ref{theorem:iv-convergence} is stated with a constant probability of failure, this is only for simplicity of presentation; in fact, the probabilities of failure all decay exponentially with $n$, once $n$ exceeds the sample complexity stated in the theorem.
\end{remark}

\section{Proof of Theorem~\ref{theorem:nonlinear-iv-convergence}}\label{appendix:nonlinear-iv-convergence}

Let $(X_i, Y_i, Z_i)$ be $n$ independent samples drawn according to $(X, Y, Z)$. Let $\epsilon>0$. We prove that under the above assumptions, if $n$ is sufficiently large, then with high probability, for any $\epsilon$-contamination $(X'_i, Y'_i, Z'_i)_{i=1}^n$ of $(X_i, Y_i, Z_i)_{i=1}^n$, the functions $g_i(w) = Z_i'(Y_i' - G((X_i')^T w))$ satisfy Assumption~\ref{assumption:gmm}. The proof is similar to the previous section, with slight complications introduced by the non-linearity of the non-linear function $G$.

\begin{lemma}\label{lemma:nonlinear-lipschitz-gradient}
Let $\epsilon > 0$. Suppose that $n \geq C(p+d)^5\log((p+d)/\epsilon)/\epsilon^4$ for an appropriate constant $C$. Then with probability at least $0.97$, there is a set $I \subseteq [n]$ with $|I| \geq (1-\epsilon/100)n$ such that for every $S \subseteq I$ with $|S| \geq (1-\epsilon)n$, it holds that \[\norm{\EE_S ZX^T G'(X^T w) - \EE_S ZX^T G'(X^T w^*)}_\text{op} \leq O(\sqrt{\tau L^3} \norm{w-w^*}_2).\]
\end{lemma}

\begin{proof}
Let $I_1$ be the set guaranteed by Lemma~\ref{lemma:linear-gradient-upper} with parameter $\epsilon/200$, and let $I_2$ be the set guaranteed by applying Corollary~\ref{corollary:empirical-covariance} to $X_1,\dots,X_n$ with parameter $\epsilon/200$. Take $I = I_1 \cap I_2$, so that $|I| \geq (1-\epsilon/100)n$. Let $S \subseteq I$ with $|S| \geq (1-\epsilon)n$. By Cauchy-Schwarz, we have that
\begin{align*}
\norm{\EE_S ZX^T (G'(X^T w) - G'(X^T w^*))}_\text{op}
&= \sup_{\norm{u}=\norm{v}=1} \EE_S u^T Z X^T v (G'(X^T w) - G'(X^T w^*)) \\
&\leq \sup_{\norm{u}=\norm{v}=1} \sqrt{\EE_S (u^T Z X^T v)^2 \EE_S (G'(X^T w) - G'(X^T w^*))^2}.
\end{align*}
First, by the guarantee of Lemma~\ref{lemma:linear-gradient-upper}, we have \[\EE_S (u^T Z X^T v)^2 \leq O(\tau L^2).\]
Second, by the guarantee of Corollary~\ref{corollary:empirical-covariance} and Lipschitzness of $G'$, we have \[\EE_S (G'(X^T w) - G'(X^T w^*))^2 \leq \EE_S (X^T (w - w^*))^2 \leq O(L \norm{w-w^*}_2^2).\]
Together, \[\norm{\EE_S ZX^T (G'(X^T w) - G'(X^T w^*))}_\text{op} \leq O(\sqrt{\tau L^3} \norm{w-w^*}_2^2)\] as claimed.
\end{proof}

\begin{lemma}\label{lemma:nonlinear-gradient-upper}
Let $\epsilon > 0$ and suppose that $n \geq C(p+d)^5\log((p+d)/\epsilon)/\epsilon^2$ for an appropriate constant $C$. Then with probability $0.98$, there is a set $I \subseteq [n]$ with $|I| \geq (1-\epsilon)n$ such that \[\EE_I(u^T ZX^T G'(X^T w) v)^2 \leq O(\tau L^2)\] for all $w \in \RR^d$ and unit vectors $u,v$.
\end{lemma}

\begin{proof}
Let $I$ be the set guaranteed by Lemma~\ref{lemma:linear-gradient-upper}. Simply note that since $G$ is $1$-Lipschitz, \[\EE_I(u^T Z X^T G'(X^T w) v)^2 \leq \EE(u^T Z)^2(X^T v)^2 \leq O(\tau L^2)\] for all unit vectors $u,v$.
\end{proof}

\begin{lemma}
Let $\epsilon>0$. Suppose that $n \geq C(p+d)^5\tau \log((p+d)/\tau\epsilon)/\epsilon^2$ for an appropriate constant $C$. There is a set $I \subseteq [n]$ with $|I| \geq (1-\epsilon/50)n$ such that for every $S \subseteq I$ with $|S| \geq (1-\epsilon)n$, it holds that \[\sigma_\text{min}(\EE_S ZX^T G'(X^T w^*)) \geq \lambda/16.\]
\end{lemma}

\begin{proof}
Let $I_1$ be the set guaranteed by Lemma~\ref{lemma:linear-least-singular} with parameter $\epsilon$, and let $I_2$ be the set guaranteed by Lemma~\ref{lemma:nonlinear-lipschitz-gradient} with parameter $\epsilon$. Let $I = I_1 \cap I_2$, so that $|I| \geq (1-\epsilon/50)n$. Pick any $S \subseteq I$ with $|S| \geq (1-\epsilon)n$. Then $\sigma_\text{min}(\EE_S ZX^T G'(X^T 0)) \geq \lambda/8$ by Lemma~\ref{lemma:linear-least-singular} (and since $G'(0) = 1/4$), and $\norm{\EE_S ZX^T (G'(0) - G'(X^T w^*))}_\text{op} \leq O(\tau L^2 \norm{w^*}_2)$ by Lemma~\ref{lemma:nonlinear-lipschitz-gradient}. It follows that \[\sigma_\text{min}(\EE_S ZX^T G'(X^T w^*)) \geq \lambda/8 - O(\sqrt{\tau L^3} \norm{w^*}_2) \geq \lambda/16,\] where the last inequality is by assumption that $\norm{w^*}_2 \leq R_0 \leq O(\lambda/\sqrt{\tau L^3}))$.
\end{proof}

\begin{lemma}\label{lemma:nonlinear-true-variance}
Let $\epsilon>0$ and suppose that $n \geq C(p+d)^3/\epsilon^2$. Then with probability $0.99$, there is a set $I \subseteq [n]$ with $|I| \geq (1-\epsilon)n$ such that \[\EE_I (v^T Z\xi)^2 \leq O(\sigma^2 L)\] for all unit vectors $v$.
\end{lemma}

\begin{proof}
By assumption, $\EE(v^T Z\xi)^2 = \EE (v^T Z)^2 (Y - G(X^T w^*))^2 \leq \sigma^2 L$. So we can apply Corollary~\ref{corollary:empirical-covariance} to conclude.
\end{proof}

\begin{lemma}\label{lemma:nonlinear-true-norm}
Let $\epsilon > 0$, and suppose that $n \geq C(p^{3/2}/\epsilon)\log(p)$. With probability $0.99$, there is a subset $I \subseteq [n]$ with $|I| \geq (1-\epsilon/100)n$ such that for every $S \subseteq I$ with $|S| \geq (1-\epsilon)n$, it holds that \[\norm{\EE_S Z\xi}_2 \leq O(\sigma \sqrt{L\epsilon}).\]
\end{lemma}

\begin{proof}
Observe that $\EE Z\xi = 0$ and $\EE ZZ^T \xi^2 \preceq \sigma^2 L I$ by assumption. The claim follows from Lemma~\ref{lemma:mean-zero-concentration}.
\end{proof}

As a result of the above lemmas, we get the following corollary, just as in the previous section.

\begin{corollary}\label{corollary:nonlinear-assumption}
Let $\epsilon>0$. Suppose that $\epsilon < c\min(\lambda^2/(\tau L^2), \lambda^4/L^4)$ for a sufficiently small constant $c>0$, and suppose that $n \geq C(d+p)^5 \tau \log((p+d)/\tau\epsilon)/\epsilon^2$ for a sufficiently large constant $C$. Suppose that $R_0 \leq c\min(\lambda^2/L,\lambda/(\tau L^2))$. Then with probability at least $0.95$, there is a set $I \subseteq [n]$ with $|I| \geq (1-\epsilon/2)n$ such that for every subset $S \subseteq I$ with $|S| \geq (1-\epsilon)n$, the following hold:
\begin{itemize}
    \item $\sigma_\text{min}(\EE_S \Gradient g(w^*)) \geq \Omega(\lambda)$
    \item $\EE_S (u^T \Gradient g(w) v)^2 \leq O(\tau L^2)$ for all unit vectors $u,v$ and all $w \in B_{R_0}(0)$
    \item $\EE_S (v^T g(w^*))^2 \leq O(\sigma^2 L)$ for all unit vectors $v$
    \item $\norm{\EE_S g(w^*)}_2 \leq O(\sigma \sqrt{L\epsilon})$
    \item $\norm{\EE_S \Gradient g(w) - \EE_S \Gradient g(w^*)}_2 \leq O(\sqrt{\tau L^3} \norm{w-w^*}_2)$ for all $w \in B_{R_0}(0)$
\end{itemize}
\end{corollary}

This corollary validates Assumption~\ref{assumption:gmm} for logistic instrumental variables, and proves Theorem~\ref{theorem:nonlinear-iv-convergence}.

\section{Technical lemmas}

In this section we collect technical lemmas that are needed for our proof. Most of these results are standard in the robust statistics literature (see, e.g., \cite{jambulapati2021robust}).

The following fact is key to the filtering algorithm and various other bounds.

\begin{lemma}\label{lemma:key}
Let $P, Q$ be distributions on $\RR^d$. Let $\epsilon \in [0, 1/2)$ and suppose that $\TV(P,Q) = \epsilon$ and $\norm{\Cov_P}_\text{op}, \norm{\Cov_Q}_\text{op} \leq \sigma^2$. Then if $X \sim P$ and $Y \sim Q$, it holds that \[\norm{\EE X - \EE Y}_2 \leq C_\epsilon\sigma\sqrt{\epsilon}\] where $C_\epsilon = \sqrt{6/(1-4\epsilon^2)}$.
\end{lemma}

\begin{proof}
Since $\TV(P,Q) = \epsilon$ there is some coupling under which $\Pr(X\neq Y) = \epsilon$. As a result, \[\EE[X] - \EE[Y] = \epsilon(\EE[X|X\neq Y] - \EE[Y|X\neq Y]).\]
Thus we have that:
\begin{align*}
\norm{\EE X-\EE Y}_2^2 =~& \epsilon^2 \norm{\EE[X|X\neq Y] - \EE[Y|X\neq Y]}_2^2\\
\leq~& \epsilon^2 \sup_{v\in \RR^d: \norm{v}_2=1} \left(v\cdot \left(\EE[X|X\neq Y] - \EE[Y|X\neq Y]\right)\right)^2
\end{align*}

Let $v \in \RR^d$ be a unit vector. Bounding the means of $X|X\neq Y$ and $Y|X\neq Y$ by second moments around $\EE X$, we have that
\begin{align*}
(\EE[v\cdot X|X\neq Y] - \EE[v\cdot Y|X\neq Y])^2
&= (\EE[v\cdot(X - \EE X)|X \neq Y] - \EE[v\cdot(Y-\EE X)|X \neq Y])^2 \\
&\leq 2\EE[v\cdot(X-\EE X)|X\neq Y]^2 + 2\EE[v\cdot(Y-\EE X)|X\neq Y]^2 \\
&\leq 2\EE[(v\cdot(X-\EE X))^2|X\neq Y] + 2\EE[(v\cdot(Y - \EE X))^2|X\neq Y]
\end{align*}
By law of total probability, \[\EE[(v\cdot(X-\EE X))^2|X\neq Y] \leq \epsilon^{-1}\EE[(v\cdot(X-\EE X))^2] = \epsilon^{-1} v^T\Cov(X)v \leq \sigma^2/\epsilon.\]
Similarly,
\begin{align*}
\EE[(v\cdot(Y - \EE X))^2|X \neq Y]
&\leq 2\EE[(v\cdot(Y - \EE Y))^2|X\neq Y] + 2(v\cdot(\EE Y - \EE X))^2 \\
&\leq 2\epsilon^{-1} \EE[(v\cdot(Y - \EE Y))^2] + 2\norm{\EE Y - \EE X}_2^2 \\
&\leq 2\sigma^2/\epsilon + 2\norm{\EE Y - \EE X}_2^2.
\end{align*}
As a result, we get that:
% , bounding the means of $X|X\neq Y$ and $Y|X\neq Y$ by second moments around $\EE X$, we have that
\begin{align*}
(\EE[v\cdot X|X\neq Y] - \EE[v\cdot Y|X\neq Y])^2
% &= (\EE[v\cdot(X - \EE X)|X \neq Y] - \EE[v\cdot(Y-\EE X)|X \neq Y])^2 \\
% &\leq 2\EE[v\cdot(X-\EE X)|X\neq Y]^2 + 2\EE[v\cdot(Y-\EE X)|X\neq Y]^2 \\
% &\leq 2\EE[(v\cdot(X-\EE X))^2|X\neq Y] + 2\EE[(v\cdot(Y - \EE X))^2|X\neq Y]^2 \\
&\leq 6\sigma^2/\epsilon + 4\norm{\EE Y - \EE X}_2^2.
\end{align*}
We conclude that 
\begin{align*}
\norm{\EE X-\EE Y}_2^2
\leq~& 6\sigma^2\epsilon + 4\epsilon^2 \norm{\EE X - \EE Y}_2^2
\end{align*}

Re-arranging we get the desired inequality.
\end{proof}

The above lemma implies that if an adversary is allowed to corrupt an $\epsilon$-fraction of data, and the original distribution has variance no more than $\sigma^2$ in any direction, then the corrupted mean must be within $O(\sigma\sqrt{\epsilon})$ of the original mean, unless the corrupted distribution has significantly larger variance.

\begin{lemma}\label{lemma:covariance-upper-concentration}
Let $\epsilon,\delta>0$. Suppose that $X_1,\dots,X_n,X$ are independent and identically distributed with $\EE XX^T = I_d$. Suppose that $n \geq Cd^3\log(3/\delta)/\epsilon^2$. Then with probability $0.99$ there is a subset $I \subseteq [n]$ with $|I| \geq (1-\epsilon)n$ such that \[\frac{1}{n} \sum_{i\in I} X_i X_i^T \preceq (1+\delta)I_d\]
and as a consequence \[\frac{1}{|I|}\sum_{i\in I} X_i X_i^T \preceq \frac{1+\delta}{1-\epsilon}I_d.\]
\end{lemma}

\begin{proof}
Since $\EE \norm{X}_2^2 = \text{Tr}(I_d) = d$, we have that $\Pr[\norm{X}_2^2 \geq 2d/\epsilon] \leq \epsilon/2$. Define $I = \{i \in [n]: \norm{X_i}_2^2 \leq 2d/\epsilon\}$. By a Chernoff bound, we have $|I| \geq (1-\epsilon)n$ with probability $1-\exp(-\Omega(\epsilon n))$. Fix a unit vector $u \in \RR^d$ and define \[A_i = \langle X_i, u \rangle^2 \mathbbm{1}[\norm{X_i}_2^2 \leq 2d/\epsilon]\] for $i \in [n]$. We have that $\EE[A_i] \leq \EE \langle X_i,u\rangle^2 = 1$, and also $A_1,\dots,A_n$ are independent and uniformly bounded by $2d/\epsilon$. Thus, Hoeffding's inequality implies that \[\Pr\left[\frac{1}{n}\sum_{i=1}^n A_i \geq 1+\delta\right] \leq \exp(-2n\delta^2/(2d/\epsilon)^2).\]
Define \[f(u) = \frac{1}{n}\sum_{i\in I} \langle X_i,u\rangle^2 = \frac{1}{n}\sum_{i=1}^n A_i.\]
For any fixed unit vector $u$ we've shown that $f(u) \leq 1+\delta$ with probability $1-\exp(-\Omega(n\delta^2\epsilon^2/d^2))$. Let $\mathcal{N}$ be a net of the unit ball in $\RR^d$ with resolution $\alpha$ and cardinality at most $(3/\alpha)^d$. By a union bound, it holds that $f(u) \leq 1+\delta$ for all $u \in \mathcal{N}$ with probability $1-\exp(d\log(3/\alpha) - \Omega(n\delta^2\epsilon^2/d^2))$. But now \[|f(u)-f(v)| \leq \frac{1}{n}\sum_{i\in I} |\langle X_i,u-v\rangle| \cdot |\langle X_i,u+v\rangle| \leq \sqrt{f(u-v)f(u+v)|}\] for any vectors $u,v$. Define \[M = \norm{\frac{1}{n}\sum_{i \in I} X_i X_i^T}_\text{op} = \max_{\norm{u}=1} f(u).\] Then \[M \leq 1 + \delta + \sqrt{\alpha M \cdot 2M}.\]
Taking $\alpha = \delta^2/2$, we get that $M \leq (1+\delta)/(1-\delta) \leq 1+4\delta$. So long as $n \geq Cd^3\log(3/\delta)/\epsilon^2$ for a large enough constant $C$, it holds with probablity at least $0.99$ that \[\frac{1}{n}\sum_{i \in I} X_i X_i^T \preceq (1+4\delta)I_d\] and moreover $|I| \geq (1-\epsilon)n$. By the latter inequality it also follows that \[\frac{1}{|I|} \sum_{i \in I} X_i X_i^T \preceq \frac{1+4\delta}{1-\epsilon}I_d\] as claimed.
\end{proof}

\begin{lemma}\label{lemma:covariance-lower-concentration}
Let $\epsilon,\tau>0$. Suppose that $X_1,\dots,X_n,X$ are independent and identically distributed with $\EE XX^T = I_d$. Suppose that $\EE\langle u,X\rangle^4 \leq \tau(\EE \langle u,X\rangle^2)^2$ for all $u \in \RR^d$. Suppose that $n \geq Cd\sqrt{\tau}\log(1/(\tau\epsilon))/\epsilon^{3/2}$ for an appropriate absolute constant $C$. Then with probability $0.99$, there is a subset $I \subseteq [n]$ with $|I| \geq (1-\epsilon/100)n$ such that for any $S \subseteq I$ with $|S| \geq (1-\epsilon)n$ it holds that \[\frac{1}{|S|} \sum_{i\in S} X_i X_i^T \succeq (1-7\sqrt{\tau\epsilon})I_d.\]
\end{lemma}

\begin{proof}
Let $I \subseteq [n]$ be the subset guaranteed by Lemma~\ref{lemma:covariance-upper-concentration}, with the properties that $|I| \geq (1-\epsilon/100)n$ and $\frac{1}{|I|}\sum_{i\in I} X_iX_i^T \preceq (1+\epsilon)I_d$.

Fix a unit vector $u \in \RR^d$. Let $q$ be such that $\Pr(\langle X_i, u \rangle^2 \geq q) = 4\epsilon$. Define $B_i = \mathbbm{1}[\langle X_i,u\rangle^2 \geq q]$. By a Chernoff bound, it holds with probability $1 - \exp(-\Omega(\epsilon n))$ that $\sum_{i=1}^n B_i \geq \epsilon n$. Thus the size of the set $Q=\{i\in [n]: \mathbbm{1}[\langle X_i,u\rangle^2 < q]\}$ is at most $(1-\epsilon)n$. As a result, any $S \subseteq [n]$ with $|S| \geq (1-\epsilon)n$, must either contain all elements from the set $Q$ or elements from its complement, whose values $\langle X_i,u\rangle^2$ dominate the value of any element in $Q$. More formally: note that $|S\cap Q| + |Q-S| = |Q| \leq |S| = |S\cap Q| + |S\cap Q^c| \implies |Q-S| \leq |S\cap Q^c|$. Since every element in $S\cap Q^c$ has value $\langle X_i,u\rangle^2$ larger than any element in $Q-S$, we thus have: $\sum_{i\in S\cap Q^c} \langle X_i,u\rangle^2 \geq \sum_{i\in Q-S} \langle X_i,u\rangle^2$. Thus, it holds that \[\sum_{i\in S} \langle X_i, u\rangle^2 \geq \sum_{i\in Q} \langle X_i, u\rangle^2 = \sum_{i=1}^n \langle X_i, u \rangle^2 \mathbbm{1}[\langle X_i, u\rangle^2 < q].\]
Next, note that $\langle X_i, u\rangle^2\mathbbm{1}[\langle X_i,u\rangle^2 < q]$ is bounded by $q^2$. Since \[4\epsilon = \Pr(\langle X,u\rangle^2 \geq q) \leq q^{-2} \EE\langle X,u\rangle^4 \leq q^{-2}\tau \left(\EE \langle X,u\rangle^2\right)^2 \leq \tau/q^2,\] we have that $q^2 \leq \tau/(4\epsilon)$. Therefore by Bernstein's inequality, with probability \[1- \exp\left(-\Omega(n \frac{\tau \epsilon}{\EE[\langle X_i, u\rangle^4] + q^2 \sqrt{\tau \epsilon}})\right)= 1- \exp\left(-\Omega(n \frac{\tau \epsilon}{\tau + \frac{\tau}{4\epsilon} \sqrt{\tau \epsilon}})\right)= 1- \exp\left(-\Omega(n \frac{\epsilon^{3/2}}{\sqrt{\tau}})\right)\]
we have that \[\frac{1}{n}\sum_{i=1}^n \langle X_i,u\rangle^2 \mathbbm{1}[\langle X_i,u\rangle^2 < q] \geq \EE \langle X,u\rangle^2 \mathbbm{1}[\langle X,u\rangle^2 < q] - \sqrt{\tau\epsilon}.\]
But now
\begin{align*}
\EE \langle X,u\rangle^2 \mathbbm{1}[\langle X,u\rangle^2 < q]
&= \EE \langle X,u\rangle^2 - \EE \langle X,u\rangle^2 \mathbbm{1}[\langle X,u\rangle^2 \geq q] \\
&\geq 1 - \sqrt{\EE \langle X,u\rangle^4 \Pr(\langle X,u\rangle^2 \geq q)} \\
&\geq 1 - \sqrt{4\tau\epsilon}.
\end{align*}
Thus, with probability $1-\exp(-\Omega(\epsilon n)) - \exp(-\Omega(n\epsilon^{3/2}/\sqrt{\tau}))$, for all $S \subseteq [n]$ with $|S| \geq (1-\epsilon)n$, we have that \[\sum_{i\in S} \langle X_i,u\rangle^2 \geq (1 - 3\sqrt{\tau\epsilon})n.\]
Assume moreover that $S \subseteq I$. Define $f(u) = \sum_{i\in S} \langle X_i,u\rangle^2$. Then for any vectors $u,v$, we have by Cauchy-Schwarz that \[|f(u)-f(v)| = \sum_{i\in S} \langle X_i,u\rangle^2 - \langle X_i,v\rangle^2 \leq \sqrt{f(u-v)f(u+v)}.\]
Since $S\subseteq I$ we have that $\sum_{i\in S} X_iX_i^T \preceq 2nI_d$. So \[|f(u)-f(v)| \leq 2n \norm{u-v}_2 \norm{u+v}_2.\]
Fix a net on the unit sphere in $\RR^d$, with resolution $\sqrt{\tau\epsilon}$ and cardinality $(O(1)/\sqrt{\tau\epsilon})^d$. Then with probability $1 - \exp(O(d\log(1/(\tau\epsilon))) - \Omega(n\epsilon^{3/2}/\sqrt{\tau}))$ the lower bound holds for all $u$ in the net and all $S \subseteq I$ of size $|S| \geq (1-\epsilon)n$. As a result, for any unit vector $v \in \RR^d$ and any such $S$, it holds that \[\sum_{i \in S} \langle X_i,u\rangle^2 \geq (1-3\sqrt{\tau\epsilon})n - 4n\sqrt{\tau\epsilon}.\]
We conclude that \[\frac{1}{|S|}\sum_{i \in S} X_i X_i^T \succeq (1-7\sqrt{\tau\epsilon})I_d.\]
So long as $n \geq Cd\sqrt{\tau}\log(1/(\tau\epsilon))/\epsilon^{3/2}$ for a sufficiently large constant $C$, this holds with probability at least $0.99$ as claimed.
\end{proof}

\begin{corollary}\label{corollary:empirical-covariance}
Let $\epsilon,\tau>0$ be sufficiently small. Suppose that $X_1,\dots,X_n$ are independent and identically distributed $d$-dimensional random vectors, with positive-definite covariance $\EE XX^T = \Sigma$. Suppose that $\EE \langle u,X\rangle^4 \leq \tau(\EE\langle u,X\rangle^2)^2$ for all $u$. Suppose that $n \geq Cd^3\sqrt{\tau}\log(1/\tau\epsilon)/\epsilon^2$ for a large constant $C$. Then with probability $0.99$ there is a subset $I \subseteq [n]$ with $|I| \geq (1-\epsilon/100)n$ such that for every subset $S \subseteq I$ with $|S| \geq (1-\epsilon)n$, it holds that \[(1-O(\sqrt{\tau\epsilon}))\Sigma \preceq \frac{1}{|S|} \sum_{i\in S} X_iX_i^T \preceq (1+O(\epsilon))\Sigma.\]
\end{corollary}

\begin{proof}
We apply Lemmas~\ref{lemma:covariance-upper-concentration} and~\ref{lemma:covariance-lower-concentration} to $\Sigma^{-1/2}X_1,\dots,\Sigma^{-1/2}X_n$. For the upper bound, we observe that if it holds for $I$ then it holds for every large subset $S$ with only an additional factor of $1+O(\epsilon)$. For the lower bound, we note that hypercontractivity is preserved under this linear transformation.
\end{proof}

\begin{lemma}\label{lemma:mean-zero-concentration}
Let $\epsilon,\sigma>0$. Let $X_1,\dots,X_n,X$ be i.i.d. $d$-dimensional random vectors with $\EE X = 0$ and $\EE XX^T \preceq \sigma^2 I$. If $n \geq C (d^{3/2}/\epsilon)\log(d)$ for a sufficiently large constant $C$, then with probability at least $0.99$, there is a subset $I \subseteq [n]$ with $|I| \geq (1-\epsilon/100)n$ such that for every $S \subseteq I$ with $|S| \geq (1-\epsilon)n$, it holds that $\norm{\EE_S X}_2 \leq O(\sigma\sqrt{\epsilon})$.
\end{lemma}

\begin{proof}
Since $\EE\norm{X}_2^2 \leq \sigma^2 d$, we have that $\Pr[\norm{X}_2^2 \geq 200\sigma^2d/\epsilon] \leq \epsilon/200$. Define $I = \{i \in [n]: \norm{X_i}_2^2 \leq 200\sigma^2d/\epsilon\}$. By a Chernoff bound, we have $|I| \geq (1-\epsilon/100)n$ with probability $1-\exp(-\Omega(\epsilon n))$. Now \[\EE XX^T \mathbbm{1}[\norm{X}_2^2 \geq 200\sigma^2 d/\epsilon] \preceq \EE XX^T \preceq \sigma^2 I,\]
and the random variables $X_i X_i^T \mathbbm{1}[\norm{X_i}_2^2 \geq 200\sigma^2d/\epsilon]$ are independent and bounded in operator norm by $200\sigma^2 d/\epsilon$. Thus, we can apply the Matrix Chernoff bound \cite{tropp2015introduction} to get
\begin{equation}\Pr\left[\frac{1}{n}\sum_{i \in I} X_i X_i^T \preceq 2e\sigma^2 I\right] \geq 1 - d\exp(-2e\sigma^2 n (\epsilon / 200\sigma^2 d) \log(2)) \geq 0.999 \label{eq:matrix-chernoff}\end{equation}
so long as $n \geq C(d/\epsilon)\log(d)$ for a sufficiently large constant $C$. Moreover, for any unit vector $v$,
\begin{align*}
\EE v^T X \mathbbm{1}[\norm{X}_2^2 \leq 200\sigma^2 d/\epsilon]
&= -\EE v^T X \mathbbm{1}[\norm{X}_2^2 > 200\sigma^2 d/\epsilon] \\
&\leq \sqrt{\EE(v^T X)^2 \Pr(\norm{X}_2^2 > 200\sigma^2 d/\epsilon)} \\
&\leq \sigma \sqrt{\epsilon}.
\end{align*}

Since $X\mathbbm{1}[\norm{X}_2^2 \leq 200\sigma^2 d/\epsilon]$ is bounded in norm by $\sigma \sqrt{200d/\epsilon}$, a Bernstein bound implies that for any unit vector $v$,
\begin{align*}
\Pr\left(\left|v \cdot \frac{1}{n}\sum_{i=1}^n X_i\mathbbm{1}[\norm{X_i}_2^2 \leq 200\sigma^2d/\epsilon]\right| > 1.5\sigma \sqrt{\epsilon}\right) \leq~& \exp\left(-\Omega\left(\frac{n\sigma^2 \epsilon}{\EE[(v^T X)^2] +  (\sigma\sqrt{d/\epsilon})(\sigma\sqrt{\epsilon})}\right)\right).\\
\leq~& \exp\left(-\Omega\left(\frac{n\sigma^2 \epsilon}{\sigma^2 +  \sigma^2 \sqrt{d}}\right)\right).
\end{align*}
Take a net over unit vectors in $\RR^d$ of granularity $1/100$ and cardinality $\exp(O(d))$. Then the above inequality holds for all $v$ in the net, with probability $\exp(O(d) - \Omega(n\epsilon/(\sqrt{d}))$, which is at least $0.999$ if $n \geq Cd^{3/2}/\epsilon$ for an appropriate constant $C$.

Let $N$ denote the aforementioned net of the unit ball in $\RR^d$. We have that in the aforementioned event:
\begin{align*}
    \norm{\frac{1}{n}\sum_{i=1}^n X_i\mathbbm{1}[\norm{X_i}_2^2 \leq 200\sigma^2d/\epsilon]}_2 =~& \sup_{w\in \RR^d: \norm{w}_2=1} \left|w \cdot \frac{1}{n}\sum_{i=1}^n X_i\mathbbm{1}[\norm{X_i}_2^2 \leq 200\sigma^2d/\epsilon]\right|\\ 
    \leq~& \sup_{v\in N} \left|w \cdot \frac{1}{n}\sum_{i=1}^n X_i\mathbbm{1}[\norm{X_i}_2^2 \leq 200\sigma^2d/\epsilon]\right|\\
    &~~ + \frac{1}{100} \norm{\frac{1}{n}\sum_{i=1}^n X_i\mathbbm{1}[\norm{X_i}_2^2 \leq 200\sigma^2d/\epsilon]}_2
\end{align*}
Re-arranging yields:
\begin{align*}
    \norm{\frac{1}{n}\sum_{i=1}^n X_i\mathbbm{1}[\norm{X_i}_2^2 \leq 200\sigma^2d/\epsilon]}_2 \leq \frac{100}{99} 1.5 \sigma \sqrt{\epsilon} \leq 2\sigma \sqrt{\epsilon}
\end{align*}

% Pick any unit vector $w \in \RR^d$ maximizing \[\left|w \cdot \frac{1}{n}\sum_{i=1}^n X\mathbbm{1}[\norm{X}_2^2 \leq 200\sigma^2d/\epsilon]\right|.\] There is some $v$, an element of the net, with $\norm{v-w}_2 \leq 1/100$. So we have that 
% \begin{align*}
% &\left|w \cdot \frac{1}{n}\sum_{i=1}^n X\mathbbm{1}[\norm{X}_2^2 \leq 200\sigma^2d/\epsilon]\right| \\
% &\qquad\leq \left|v \cdot \frac{1}{n}\sum_{i=1}^n X\mathbbm{1}[\norm{X}_2^2 \leq 200\sigma^2d/\epsilon]\right| + \left|(w-v) \cdot \frac{1}{n}\sum_{i=1}^n X\mathbbm{1}[\norm{X}_2^2 \leq 200\sigma^2d/\epsilon]\right|\\
% &\qquad\leq 1.5\sigma\sqrt{\epsilon} + \frac{1}{100}\left|w \cdot \frac{1}{n}\sum_{i=1}^n X\mathbbm{1}[\norm{X}_2^2 \leq 200\sigma^2d/\epsilon]\right|
% \end{align*}
% where the last inequality is since $w$ was chosen as a maximizer over the unit sphere, and by linearity. Grouping terms, it follows that \[\left|w \cdot \frac{1}{n}\sum_{i=1}^n X\mathbbm{1}[\norm{X}_2^2 \leq 200\sigma^2d/\epsilon]\right| \leq 2\sigma\sqrt{\epsilon}.\]

Therefore \begin{equation}\Pr\left(\norm{\frac{1}{n}\sum_{i=1}^n X_i\mathbbm{1}[\norm{X_i}_2^2 \leq 200\sigma^2d/\epsilon]}_2 \leq 2\sigma\sqrt{\epsilon}\right) \geq 0.999\label{eq:trunc-norm}\end{equation}
In the intersection of the above events described by Equations~\ref{eq:matrix-chernoff} and~\ref{eq:trunc-norm}, and the event that $|I| \geq (1-\epsilon/100)n$, which together occur with probability at least $0.99$, we get that $\Cov_I(X) \preceq 4e\sigma^2 I$ and $\norm{\EE_I X}_2 \leq 4\sigma\sqrt{\epsilon}$. By Lemma~\ref{lemma:key}, for any $S \subseteq I$ with $|S| \geq (1-\epsilon)n$, it holds that $\norm{\EE_I X - \EE_S X}_2 \leq O(\sigma \sqrt{\epsilon})$ so in fact $\norm{\EE_S X}_2 \leq O(\sigma\sqrt{\epsilon})$.
\end{proof}

\begin{lemma}\label{lemma:fourth-moment-concentration}
Let $\epsilon>0$. Let $X_1,\dots,X_n,X$ be independent and identically distributed $d$-dimensional random vectors with $\EE \langle X,u\rangle^4 \leq \norm{u}_2^4$ for all $u \in \RR^d$ and coordinate-wise bounded $8$-th moments, i.e. $\max_{i=1}^d \EE X_i^8 \leq C_8$. Suppose that $n \geq Cd^5\log(d/\epsilon)/\epsilon^2$ for a sufficiently large constant $C$. With probability at least $0.99$, there is a set $I \subseteq [n]$ with $|I| \geq (1-\epsilon)n$ such that \[\frac{1}{|I|} \sum_{i \in I} \langle X_i, u \rangle^4 \leq c\norm{u}_2^4\] for all $u \in \RR^p$ and an absolute constant $c$.
\end{lemma}

\begin{proof}
Since $\EE \norm{X}_2^2 = \text{Tr}(\EE XX^T) \leq d$ (since $(\EE \langle u,X\rangle^2)^2 \leq \EE \langle X,u\rangle^4 \leq 1$ for any unit vector $u$), we have that
\[\Pr[\norm{X}_2^2 \geq 2d/\epsilon] \leq \epsilon/2.\]
By a Chernoff bound, we have that $|\{i \in [n]: \norm{X_i}_2^2 \geq 2d/\epsilon\}| \leq \epsilon n$ with probability at least $1 - \exp(-\Omega(\epsilon n))$. Now fix a unit vector $u \in \RR^d$ and define \[A_i = \langle X_i, u \rangle^4 \mathbbm{1}[\norm{X_i}_2^2 \leq 2d/\epsilon]\] for $i \in [n]$. We have that \[\EE[A_i] \leq \EE \langle X_i, u \rangle^4 \leq 1\]
and also $A_1,\dots,A_n$ are independent and uniformly bounded by $(2d/\epsilon)^2$. Thus, the Bernstein bound implies that 

\[\Pr\left[\frac{1}{n}\sum_{i=1}^n A_i > 2 \right] \leq \exp\left(-c_1\frac{n}{\EE\langle X, u\rangle^8 + (2d/\epsilon)^2}\right).\]
for some universal constant $c_1$. Note that:
\begin{align*}
    \EE\langle X, u\rangle^8 &\leq \EE\norm{X}_2^8 = \EE\left(\sum_{i=1}^d X_i^2\right)^4 = d^4 \EE\left(\frac{1}{d} \sum_i X_i^2 \right)^4 \\
    &\leq d^4 \EE \frac{1}{d} \sum_i X_i^8 \leq d^4 \max_{i=1}^d \EE X_i^8 \leq d^4 C
\end{align*}
Thus:
\[\Pr\left[\frac{1}{n}\sum_{i=1}^n A_i > 2 \right] \leq \exp\left(-c_1\frac{n}{C d^4 + (2d/\epsilon)^2}\right).\]

Take $I = \{i: \norm{X_i}_2^2 \leq 2d/\epsilon\}$. For any fixed unit vector $u \in \RR^d$, it holds that $\frac{1}{n}\sum_{i \in I} \langle X_i,u\rangle^4 \leq 2$ with probability $\exp(-\Omega(n/(d^4/\epsilon^2)))$. Take $\delta = \epsilon^2/d^2$. We can union bound over a $\delta$-net of the unit ball in $\RR^d$, which has cardinality at most $(3/\delta)^d$, and note that \[\left|\frac{1}{n}\sum_{i \in I}\langle X_i, u \rangle^4 - \frac{1}{n} \sum_{i \in I} \langle X_i, v \rangle^4\right| \leq C(2d/\epsilon)^2 \norm{u-v}_2,\] so in fact it holds that \[\frac{1}{n}\sum_{i \in I} \langle X_i, u \rangle^4 \leq 2 + C(2d/\epsilon)^2 \delta \leq C'\] for all unit vectors $u \in \RR^d$, with probability \[1-\exp\left(O(d\log(d/\epsilon)) - \Omega\left(\frac{n}{C d^4 + (d/\epsilon)^2}\right)\right) \geq 0.999\] since $n \geq C' d^5\log(d/\epsilon)/\epsilon^2$ for a sufficiently large constant $C'$. Finally, it also holds that $|I| \geq (1-\epsilon)n$ with probability $1-\exp(-\Omega(\epsilon n))$. It therefore holds with probability at least $0.99$ that for all unit vectors $u \in \RR^d$, \[\frac{1}{|I|} \sum_{i \in I} \langle X_i, u \rangle^4 \leq C''\] as claimed.
\end{proof}

\end{document}